%% file: VS-loss_analysis.tex
\documentclass[conference]{IEEEtran}
\IEEEoverridecommandlockouts
\usepackage{cite}
\usepackage{amssymb,amsfonts}
\usepackage{algorithmic}
\usepackage{graphicx}
\usepackage{textcomp}
\usepackage{xcolor}
\usepackage{tikz}
\usepackage{tcolorbox,wrapfig}
\usepackage[labelfont={bf},font=small]{caption}
\usepackage{subcaption}
\usepackage[ruled,vlined]{algorithm2e}
\tcbuselibrary{skins}
\usepackage[utf8]{inputenc} 
\usepackage[T1]{fontenc}    
\usepackage{hyperref}       
\usepackage{url}            
\usepackage{booktabs}       
\usepackage{nicefrac}       
\usepackage{microtype}      

\usepackage{enumitem}

\usepackage{amsmath}

\usepackage{amsthm}

\usepackage[para]{footmisc}

\def\BibTeX{{\rm B\kern-.05em{\sc i\kern-.025em b}\kern-.08em T\kern-.1667em\lower.7ex\hbox{E}\kern-.125emX}}

\DeclareMathOperator*{\argmin}{arg\,min}

\def\cond{\; | \;}

\input{commands}

\def\BibTeX{{\rm B\kern-.05em{\sc i\kern-.025em b}\kern-.08em
    T\kern-.1667em\lower.7ex\hbox{E}\kern-.125emX}}

\DeclareRobustCommand*{\IEEEauthorrefmark}[1]{%
	\raisebox{0pt}[0pt][0pt]{\textsuperscript{\footnotesize #1}}%
}

\begin{document}

\title{
On how to avoid exacerbating spurious correlations when models are overparameterized
}




\author{
	\IEEEauthorblockN{Tina Behnia \IEEEauthorrefmark{1}, \hspace{5pt} Ke Wang \IEEEauthorrefmark{2}, \hspace{5pt} Christos Thrampoulidis\IEEEauthorrefmark{1}\thanks{This work is supported by the an NSERC Discovery Grant, by an NSF Grant  CCF2009030, and by a CRG8 award from KAUST.\newline}}
	}

\maketitle

\footnotetext[1]{Department of Electrical and Computer Engineering, University of British Columbia, Vancouver.} \footnotetext[2]{Department of Statistics and Applied Probability, University of California, Santa Barbara.}
\begin{abstract}
Overparameterized models fail to generalize well in the presence of data imbalance even when combined with traditional techniques for mitigating imbalances.
This paper focuses on imbalanced classification datasets, in which a small subset of the population---a minority---may contain features that correlate spuriously with the class label. For a parametric family of cross-entropy loss modifications and a representative Gaussian mixture model, we derive non-asymptotic generalization bounds on the worst-group error that shed light on the role of different hyper-parameters. Specifically, we prove that, when appropriately tuned, the recently proposed VS-loss learns a model that is fair towards minorities even when spurious features are strong. On the other hand, alternative heuristics, such as the weighted CE and the LA-loss, can fail dramatically.  Compared to previous works, our bounds hold for more general models, they are non-asymptotic, and, they apply even at scenarios of extreme imbalance. 
\end{abstract}


\section{Introduction}
\input{intro}

\section{Setup and background}
\input{back}

\section{Results}
\input{results}

\section{Numerical Results}
\input{num}

\section{Conclusion}
\input{conclusion}

%


\bibliographystyle{IEEEtran}
\bibliography{refs}

\newpage
\appendix
\section{Appendix}

\section{Proofs}



\input{proofs3}

\end{document}

%% file: commands.tex
\usepackage{mathtools}
\definecolor{darkred}{RGB}{150,0,0}
\definecolor{darkgreen}{RGB}{0,150,0}
\definecolor{darkblue}{RGB}{0,0,200}
\hypersetup{colorlinks=true, linkcolor=darkred, citecolor=darkgreen, urlcolor=darkblue}

\newtheorem{theorem}{Theorem}

\newtheorem{assumption}{Assumption}

\newtheorem{lemma}{Lemma}
\newtheorem{fact}{Fact}
\newtheorem{corollary}{Corollary}

\newtheorem{remark}{Remark}

\newcommand{\Rwst}{\Rc_{\text{wst}}}
\newcommand{\ellp}{\ell^\prime}

\newcommand{\Lpall}{L^\prime_{0:t,b}}
\newcommand{\Lpallneg}{L^\prime_{0:t,-b}}
\newcommand{\Lpallmaj}{L^\prime_{0:t,1}}
\newcommand{\Lpallmin}{L^\prime_{0:t,-1}}
\newcommand{\Lpt}{L^\prime_{t,b}}
\newcommand{\Lptau}{L^\prime_{\tau,b}}
\newcommand{\Lptneg}{L^\prime_{t,-b}}

\newcommand{\Lpts}{L^\prime_{t}}
\newcommand{\Lptaus}{L^\prime_{\tau}}
\newcommand{\Lpalls}{L^\prime_{0:t}}

\newcommand{\cut}[1]{\textcolor{red}{}}

\newcommand{\nub}{\boldsymbol{\nu}}
\newcommand{\q}{\mathbf{q}}

\newcommand{\x}{\mathbf{x}}

\newcommand{\w}{\mathbf{w}}

\newcommand{\eb}{\mathbf{e}}

\newcommand{\z}{\mathbf{z}}

\newcommand{\Tc}{{\mathcal{T}}}

\newcommand{\Yc}{\mathcal{Y}}

\newcommand{\Rc}{\mathcal{R}}
\newcommand{\Nn}{\mathcal{N}}
\newcommand{\Lc}{\mathcal{L}}

\newcommand{\Gc}{{\mathcal{G}}}
\newcommand{\Ac}{\mathcal{A}}
\newcommand{\Pc}{\mathcal{P}}
\newcommand{\Ec}{\mathcal{E}}

\newcommand{\beq}{\begin{equation}}
\newcommand{\eeq}{\end{equation}}
\newcommand{\bea}{\begin{align}}
\newcommand{\eea}{\end{align}}

\newcommand{\R}{\mathbb{R}}

\newcommand{\Pro}{\mathbb{P}}
\newcommand{\nn}{\notag}

  \newcommand{\mub}{\boldsymbol\mu}
    
  \newcommand{\eps}{\epsilon}

\DeclarePairedDelimiterX{\inp}[2]{\langle}{\rangle}{#1, #2}

\newcommand{\wt}{\widetilde\w}

\providecommand{\norm}[1]{\lVert#1\rVert}

%% file: intro.tex

The overparameterization trend in machine learning (i.e., training  large models whose size far exceeds that of the training set) has been popularized following repeated empirical observations 
that large models achieve state-of-the-art accuracy despite being able to perfectly fit / interpolate the training data \cite{zhang2021understanding,belkin2019reconciling,nakkiran2019deep}. These surprising empirical findings have precipitated a surge of research activity towards developing a theory of (so-called) \emph{harmless interpolation} or \emph{benign overfitting} that seeks explanations and better understanding of the interplay between large models, (gradient-based) optimization and data, e.g. \cite{liang2018just,mei2019generalization,bartlett2020benign,hastie2019surprises}. Such a theory is still at its infancy. Yet, remarkable progress has contributed to reinforcing the modern wisdom: ``large models train better and training longer improves accuracy.''
On the other hand, 
 the inherent ability of overparameterized models to perfectly fit any training objective hides risks. 
Specifically, this becomes relevant when learning from imbalanced datasets with unequal (sub)population distributions. In the context of classification, a traditional approach to mitigate data imbalances is using a weighted cross-entropy (CE) loss that up-weights the contribution of minority examples to the total loss by a factor that is proportional to the level of imbalance. However, large models' expressive ability to perfectly fit all the data---including minorities---urges us to probe the efficacy of that technique.
 Indeed, it has been shown empirically that weighted CE, and conventional techniques alike, are \emph{not} effective under overparameterization, e.g. \cite{byrd2019effect,sagawa2019distributionally,spurious}.


Thus, researchers have sought alternatives that are tailored to overparameterized models. Among  many heuristic proposals in the literature, one that showed promising performance on benchmark imbalanced datasets and attracted significant attention (perhaps also owing to its simplicity) is the so-called \emph{logit-adjusted} (LA) loss \cite{TengyuMa,Menon}. Compared to weighted CE, the LA-loss modifies the individual \emph{logits} by \emph{additive} hyperparameters that depend on the frequency of the corresponding example. While these additive adjustments were first introduced with the intention of creating larger margins for the minorities \cite{TengyuMa,Menon}, it was recently shown by \cite{VSloss}, that training overparameterized models for a large number of epochs can counter this effect (similar to countering the effect of weights in weighted CE). As a remedy, \cite{VSloss} proposed the so-called \emph{vector-scaling} (VS) loss, which instead of additive, uses \emph{multiplicative} factors to adjust the CE logits. Specifically, the authors showed that the multiplicative logit adjustments, unlike the additive ones, induce an implicit bias that favors convergence to models with larger minority margins. 

In this paper, we complement the results of \cite{VSloss} with an alternative non-asymptotic analysis of their proposed VS-loss that further emphasizes its benefits for learning from data with spurious correlations and (possibly extreme) imbalances. Specifically, inspired by \cite{spurious}, we study a (generative) Gaussian mixture model (GMM) {as shown in Figure  \ref{fig:model}}, in which  features $\x=[\x_c \,;\,\x_s]$ are split into core and spurious ones. The core features $\x_c$ are determined based on the example's label $y=\{-1,+1\}$, while the spurious features $\x_s$ are determined by an additional \emph{attribute} variable $a$, which is correlated---perhaps \emph{spuriously}---with the label. We assume that the attribute values are known during training, but are unknown at test time. When the attribute is spuriously correlated with the correct label, e.g. when $a=-y$, it naturally hinders classification. Whether the attribute agrees with the label ($a=y$) or not ($a=-y$) defines two subpopulation groups and the goal is to learn a classifier with good generalization over instances from both groups. Specifically, we want to minimize \emph{worst-group error}, rather than average error. For such a model, extensive numerical experiments in \cite{spurious} revealed that overparameterized models trained with weighted CE have poor worst-group test error  when (i) the relative sizes of two groups ($a=y$ or $a=-y$) are very different, and, (ii) the signal-to-noise ratio (SNR) of the spurious features $\x_s$ is comparable to that of the core features.  \emph{Is this an inherent limitation of training overparameterized models without explicit regularization, or, are there alternative training methods that can achieve good worst-group accuracy?}

\vspace{1pt}
\noindent\textbf{Contributions.}~
We show that an appropriately tuned VS-loss can achieve good performance despite imbalances and spurious feature correlations. We derive a non-asymptotic bound on the worst-group error of the VS-loss in a linear setting and identify overparameterized regimes in which the bound becomes exponentially small as the model size increases, even under extreme group imbalances and possibly high SNR of spurious features. Interestingly, our bound requires a proper tuning of (both the additive and multiplicative) VS-loss hyper-parameters that is consistent with the heuristic choices made in the literature \cite{Menon,CDT,VSloss,Autobalance}. To further emphasize the critical role of the multiplicative factors, we derive a lower bound on the worst-group error of the LA-loss which is \emph{non}-vanishing when groups are imbalanced and spurious correlations are significant. We further extend our positive results on the VS-loss to a noisy setting in which training labels can be corrupted. This connects our results to the benign overfitting literature, showing that interpolation of noisy data can be harmless (in terms of worst-group error) despite imbalances and spurious correlations. Finally, we show that under sufficient overparameterization the VS-loss with hyperparameters $\Delta_\pm$, learns a classifier that interpolates the data with respect to a label encoding that assigns labels ${\pm 1/\Delta_\pm}$ to respective majority ($+$) / minority ($-$) examples. For simplicity, we use an `exponential' form of the VS-loss in our analysis. Analogous simplifications have been adapted for the CE loss in the literature, e.g. \cite{chatterji,soudry2018implicit,ji2018risk,VSloss}. We believe an extension to the CE loss is mostly based on the same core argument, but we leave this for future work.

\vspace{1pt}
\noindent\textbf{Related work\footnote{The impact of spurious correlations has also been considered when studying domain generalization, that is settings in which training and test environments might differ (with respect to core and/or spurious features) e.g. \cite{arjovsky2019invariant}. Instead, in our setting the test and train samples come from the same environment.}.}~As mentioned, our study is primarily motivated by the works of Kini et al. \cite{VSloss} (in terms of algorithms) and Sagawa et al. \cite{spurious} (in terms of data model).  In terms of analysis, we follow the approach introduced by Chatterji and Long \cite{chatterji} who derived non-asymptotic bounds on the average error of CE minimization for overparameterized binary mixtures and identified regimes of benign overfitting. The key observation in their analysis, which directly tracks the error of gradient-descent (GD) iterates, is that---under sufficient overparameterization---the ratio of derivatives of the loss at any two training examples is at most a \emph{constant} factor. Thus, the influence of any one example on the GD trajectory can be well controlled. We include a detailed comparison to \cite{chatterji} and \cite{VSloss} in Sec. \ref{sec:more-related}. More broadly, our results fit in the emerging literature on overparameterized learning theory. While an exhaustive reference is out of scope, we briefly mention close connections to \cite{dengmodel,montanari2019generalization,taheri2020sharp,liang2020precise,aubin2020generalization,dhifallah2021inherent,Vidya,Ke,wang2021benign,ardeshir2021support,cao2021risk}, which all study benign overfitting in linear classification. In all these works, the crux of the analysis unfolds over two stages: (a) First, leverage optimization results on implicit bias of GD, e.g. convergence to the hard-margin support-vector machine (SVM) solution \cite{soudry2018implicit,ji2018risk}; (b) Second, perform a high-dimensional generalization analysis of the equivalent (at convergence) SVM classifier. One group of works, \cite{dengmodel,montanari2019generalization,taheri2020sharp,liang2020precise,aubin2020generalization,dhifallah2021inherent} accomplishes the second stage by using the convex Gaussian min-max theorem \cite{StoLASSO,TOH15,TAH18}. This results in exact formulas for the generalization error that hold asymptotically in the high-dimensional proportional regime. Another set of works, \cite{Vidya,Ke,wang2021benign,ardeshir2021support,cao2021risk} derive non-asymptotic bounds on the error of SVM, by relating it to the simpler (owing to its closed form expression) pseudo-inverse solution. This is made possible because they first show that, under sufficient overparameterization, all data points are support vectors. Here, we follow \cite{chatterji}. The approach does not fall into either of the above paradigms. Specifically, it does not rely on first establishing an optimization result on implicit bias of GD. Interestingly, our aforementioned finding that VS-loss interpolates the data with appropriate labeling $\pm1/\Delta_{\pm}$ resembles the findings in \cite{Vidya,Ke,wang2021benign,ardeshir2021support,cao2021risk} and creates a missing link between the different techniques in \cite{chatterji} versus \cite{Vidya,Ke,wang2021benign,ardeshir2021support,cao2021risk}.



\begin{figure}[!t]
	\centering
	\includegraphics[width=3in]{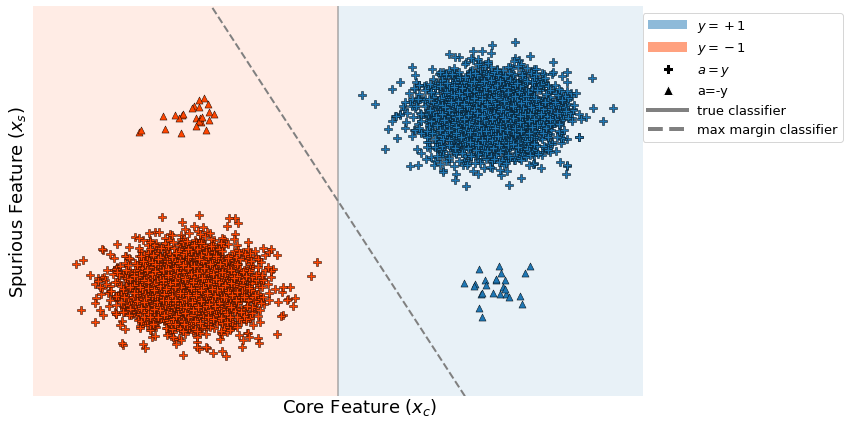}
	\caption{Illustration of a binary classification task in which features are composed of a `core' and a `spurious' part. The core part is determined by the label $y$, while the spurious part is determined by an attribute $a$ that is correlated to $y$. For a few minority samples (marked with $\triangle$), the attribute is spuriously correlated with the label (here, $a=-y$). This shifts the max-margin classifier away from the true decision boundary and towards minorities, hindering correct classification of the latter.}
	\label{fig:model}
\end{figure}


%% file: back.tex

Consider supervised classification of data $(\x,a,y)$ sampled i.i.d from a distribution $\Pc=\Pc_{\text{feature}}\times\Pc_{\text{attribute}}\times\Pc_{\text{label}}$. Here, $\x\in\R^d$ denotes the feature vector and $y$ represents the label. For simplicity, we focus on binary classification, i.e.  $y\in\Yc=\{-1,+1\}$. Each datapoint is also associated with an attribute $a\in\Ac{=\{-1,+1\}}$ that describes a subset of the features and is correlated with the label. Importantly, this correlation can be spurious. As an example, consider the task of classifying images of birds into labels $\{$`waterbird',`landbird'$\}$ \cite{sagawa2019distributionally}. If an image's background is $a=$`land background', while the object's label is $y=$`waterbird', then the strong association of the attribute with the wrong label $-y=$`landbird' can hinder the classification tasks. During training, we wish to learn a classifier that avoids overly relying on such spurious features, while at the same time results in a classification rule that performs well on all test samples irrespective of their attributes (spurious or not).  To formalize this, note that each 
%
%
%
datapoint belongs to a group $g\in\Gc=\Yc\times \Ac$. 
Thus, we seek a classifier $f_\w:\R^d\rightarrow\R$, parameterized by $\w$,  that achieves low \emph{worst-group error}:
\begin{align}\label{eq:worst}
\Rwst:=\max_{g\in\Gc} \Pro_{(\x,a,y)\sim\Pc}\left\{ f_\w(\x) y < 0 \,|\,(y,a)=g\right\}.
\end{align}
Note that in settings with \emph{group-imbalances}, i.e. when certain population groups $g=(y,a)$ are underrepresented, the {worst-group error} $\Rc_{\text wst}$, can be very different compared to the average error $\Rc_{\text{avg}}:=\Pro_{(\x,a,y)\sim\Pc}\left\{ f_\w(\x) y  < 0\right\}$.  
Also, note that $f_\w(\x)$ decides on the label of the test sample based only on the feature vector $\x$. That is, the attribute $a$ is \emph{hidden} at test time. On the other hand, we assume access to attributes during training. That is, to learn $\w$, we assume access to a train set $\Tc:=\left\{(\x_i,a_i,y_i)\right\}_{i\in[n]}$ of $n$ i.i.d samples from $\Pc$. 

For a loss function $\ell(f_\w(\x),g):\R\times\Gc\rightarrow\R$, we seek classifier $\hat f(\x):=f_{\hat\w}(\x)$  that minimizes the (un-regularized) empirical risk 
$
\Lc_n(\w) := \frac{1}{n}\sum_{i\in[n]} \ell(y_if_\w(\x_i),g_i).
$
For the minimization, we run gradient descent (GD) with constant step-size $\eta>0$ until convergence, resulting in iterations 
$
\w^{t+1} = \w^{t} -\eta \,\nabla\Lc_n(\w^t), t=0,1,\ldots.
$
 The output classifier is
$
\hat f(\x) := f_{\bar{\w}^\infty}(\x)
$ for $\bar{\w}^{\infty}:=\lim_{t\rightarrow\infty} \w^t / \|\w^t\|_2$ the classifier direction where GD converges to; (in our setting, the limit always exists.)

In this paper, we study the VS-loss proposed by \cite{VSloss} as an alternative to weighted CE tailored to overparameterized settings. The VS-loss defines a parametric family of CE losses with (hyper)parameters that not only weigh the loss itself, but also they adjust the logits. Specifically, the binary VS-loss is $$\ell_{\text{VS}}(yf_\w(\x),g)=\omega_g\log(1+e^{-\Delta_{g}y\,f_\w(\x)+\iota_g}),$$ for weigths $\omega_g>0$, and multiplicative and additive logit adjustments $\Delta_g>0$ and $\iota_g$. Note that the hyperparameters are parameterized by the group $g$. Specific choices of $(\omega_g,\iota_g,\Delta_g)$ recover the CE-loss, the weighted CE loss and the LA-loss. For example, the LA-loss of \cite{TengyuMa,Menon} corresponds to setting $\Delta_g=1$, i.e. only using additive logit adjustments. 

\subsection{A Gaussian mixture model}\label{sec:GMM}

Our analysis applies to 
a Gaussian mixture model that defines the distribution $\Pc$ of the data $(\x,a,y)$ as follows. The label $y$ takes values $\pm1$ with probabilities $\pi_{\pm1}$. The attribute $a$ also takes values $\pm1$ and is correlated to $y$. Specifically, $P\{a|y
\}=p_{a,y}$ for $(y,a)=(\pm1,\pm1).$
Finally, the feature distribution conditional on the label and attribute choice (aka, the group) is an isotropic Gaussian with mean $[y\mub_{c} \,;\, a\mub_s]\in\R^d$, i.e. 
\begin{align}\label{eq:gmm_def}
\x\,|\,(y,a) = 
\begin{bmatrix}\x_c \\ \x_s \end{bmatrix}\,|\,(y,a) \sim \Nn\Big(\begin{bmatrix}y\mub_{c} \\ a\mub_s\end{bmatrix},\mathbf{I}_d\Big).
\end{align}
Our results emphasize the specific impact of the two following model characteristics on the worst-group error.

\vspace{1pt}
\noindent\textbf{(i) Imbalance ratio.} We focus on a scenario where samples with spurious correlations, i.e. $(y,a)=(+1,-1),(-1,+1)$,  are minorities. Thus, it is more convenient to define group membership of a sample $(\x,a,y)$ according to the value $b:=y\cdot a$. This splits the data in two (instead of four) groups identified by $b=+1$ for $a=y$, and, $b=-1$ for $a=-y.$ For the training data, we let $\Tc_{b}=\{i\in[n]\,:\,y_ia_i=b\}, b=\pm1$ denote the two corresponding subsets. Letting $n_b=|\Tc_b|$ denote the size of each group within the training set, we use $\tau:=n_{1}/n_{-1}$ to denote the imbalance ratio. For a setting where the group of spurious samples ($b=-1$) is minority, we have $\tau\gg 1$.

\vspace{1pt}
\noindent\textbf{(ii) Spurious-features SNR.} The SNR of  spurious (resp. core) features is determined by the energy  $\|\mub_s\|_2^2$ (resp. $\|\mub_c\|_2^2$) of the spurious (resp. core) mean vector. We denote the total SNR as $R_+=\|\mub_c\|_2^2+\|\mub_s\|_2^2$. We also define $R_-=\|\mub_c\|_2^2-\|\mub_s\|_2^2$ and use the ratio $R_-/R_+\in[-1,1]$ to capture how strong are the core features compared to the spurious ones. Note for example that the ratio is zero when the SNR of the spurious features is same as the core features.   

\vspace{1pt}
\noindent\textbf{Extensions.} Equation \eqref{eq:gmm_def} assumes disjoint support of spurious / core features, which can be easily relaxed to $\x\sim\Nn(y\bar\mub_c+a\bar\mub_s,\mathbf{I}_{d})$: \eqref{eq:gmm_def} is a special case for $\inp{\bar\mub_c}{\bar\mub_s}=0$; we focus on this since it simplifies the exposition.
It is also relatively straightforward to extend our results to sub-Gaussian mixtures (i.e. the feature noise in \eqref{eq:gmm_def} be sub-Gaussian rather than Gaussian). Since the extension does not offer significant additional insights, we focus on the Gaussian case.

%% file: results.tex
\subsection{Worst-group error of the VS-loss}
We show that appropriately tuned VS-loss can learn from GMM data despite imbalances and spurious correlations. 


\begin{assumption}[Data model]\label{ass:means} Data (both training and test) are generated according to the GMM specified in Section \ref{sec:GMM}. 
\end{assumption}



Before stating our bound, we specify the regime of problem parameters $n,d,\mub_c,\mub_s$ and of a failure probability $\delta$ under which the results are valid.

%
%
%
%
%
%
%
%
%

\begin{assumption}[Problem regime]\label{ass:dim}
 We say Assmimption \ref{ass:dim} holds with constant $C_0$ if for \emph{all} large constants $C>C_0$:
 \begin{enumerate}[label=(\Alph*)]
 	\item $n\geq C\log(1/\delta),$
 	\item $\|\mub_{c}\|_2^2\geq C\log(n/\delta),$
 	\item $d\geq CR_+n,$
 	\item $d\geq Cn^2\log(n/\delta).$
 \end{enumerate} 
%

%
\end{assumption}

The first two conditions are rather mild. The first requires increasing sample size for higher success probability $1-\delta$ and the second that the SNR of core features grows logarithmically with $n/\delta$. The last two conditions require sufficient overparameterization. The third one, essentially asking for $d\geq C n\log(n)$, is only slightly more restrictive than the minimal overparameterization requirement of $d\geq Cn$. On the other hand, the fourth condition is more restrictive, and while we believe that it can be relaxed, it is currently needed in our proof.

For simplicity of the analysis, we assume an `exponential' analogue of the VS-loss and work with linear classifiers. For a discussion on CE loss, see Section \ref{sec:extension}.
Thus, our classifier $f_\w(\x)=\inp{\w}{\x}$ results by minimizing the empirical loss
$$
\Lc_n(\w) := 
\sum_{i\in[n]} \ell_{i,t}:=
\sum_{i\in[n]} \omega_{b_i}e^{-\Delta_{b_i}y_i\,\inp{\w}{\x_i}+\iota_{b_i}},
$$
with GD. Specifically, for an initialization $\w^0$ and step size $\eta$; we have
\begin{align}\label{eq:GD}
\w^{t+1} = \w^t + \eta\sum\nolimits_{i\in[n]}\ellp_{i,t}\, y_i\x_i\,.
\end{align}
Here, we have denoted
the \emph{negative} of the loss derivative corresponding to $(\x_i,a_i,y_i)$ at time $t$ as 
$$\ellp_{i,t}=\ell^\prime(y_i\inp{\x_i}{\w};b_i)=\Delta_{b_i}\ell_{i,t}.$$
We are interested in the \emph{worst-group error} of GD iterates as $t\rightarrow\infty$,
$$
\Rwst^\infty:=\max_{b\in\{\pm1\}} \Pro_{(\x,a,y)\sim\Pc}\left\{ y\inp{\bar{\w}^\infty}{\x} < 0 \,|\,b=y\cdot a\right\}.
$$
  In the theorem, $Q(\cdot)$ denotes the Q-function.



\begin{theorem}[VS-loss]\label{thm:main}
Assume data generated as in Assumption \ref{ass:means} with $R_-\geq0$, i.e. $\|\mub_c\|_2\geq\|\mub_s\|_2.$ Consider the worst-group error $\Rc_{wst}^\infty$ of gradient descent in \eqref{eq:GD} with initialization $\w^0=0.$ Then, there exists absolute constants $c_1>0$ and  $c=c(c_1)>0$ and $C_0:=C_0(c_1)>0$ such that the following statement holds for all problem parameters and success probability $1-\delta$ satisfying Assumption \ref{ass:dim} with constant $C_0$.
With probability at least $1-\delta$ over the realization of the training set $\Tc$,  the worst-group error is bounded by
\begin{align}\label{eq:main_bound}
\Rc_{wst, \rm{VS}}^\infty\leq Q\big( c {R_+}\big/{\sqrt{d}}\big),
\end{align}
provided  the step-size is sufficiently small,
\begin{align}\label{eq:Delta}
\frac{1}{2}\, \frac{n_{\pm}}{n} \leq \Delta_{\pm} \leq 2 \frac{n_{\pm}}{n}\quad\text{ and }\quad \frac{\omega_+ e^{\iota_+}}{\omega_- e^{\iota_-}} = \frac{\Delta_{-}^2}{\Delta_+^2}\,.
\end{align}
E.g., $\omega_{\pm}=1$ and $\iota_\pm = -2\log(\Delta_{\pm})$ is consistent with this.  

\end{theorem}

The bound in \eqref{eq:main_bound} guarantees a vanishing VS-loss error provided  $R_+=\omega(d^{1/2})$. On the other hand, Assumption \ref{ass:dim} needs $R_+=O(d/n).$ For a concrete setting, consider $n=O(d^\eps)$ for  (arbitrarily) small $\eps\geq 0$. Theroem \ref{thm:main} shows that the VS-loss error is vanishing provided that $R_+=\Theta(d^\beta)$ for $\beta\in(1/2,1-\eps].$ Importantly, this conclusion holds for \emph{any} value of imbalance ratio $\tau$. Also, spurious correlations can be tolerated provided that their SNR is controlled by that of the core features. This condition is captured by the assumption   $R_-\geq 0$. 


\subsection{Proof sketch}
\input{proofs-sketch}
\subsection{The role of the hyperparameters}

Here, we further discuss the special role of  $\Delta_\pm,\iota_\pm$ and $\omega_\pm$ in Theorem \ref{thm:main}. We also emphasize the value of Theorem \ref{thm:main} by contrasting the favorable performance of the VS-loss to that of the LA-loss.

\vspace{5pt}
\noindent\textbf{$\Delta$ mitigates imbalances and spurious correlations.}~We start by emphasizing the critical role of the $\Delta_\pm$ factors. This is formalized in the theorem below, which shows that if $\Delta_\pm=+1$ (hence, if the LA-loss is used instead), then the worst-group error can be significant in the case of extreme imbalances and strong spurious features. 

\begin{theorem}[LA-loss]\label{thm:LALoss}
	Let Assumptions \ref{ass:means} and \ref{ass:dim} hold as stated in Theorem \ref{thm:main}. Consider running GD from zero initialization on the (exponential) LA-loss (i.e. VS-loss with $\Delta_{\pm}=1$). Then, for \emph{any} choices of the hyper-parameters $\iota_\pm,\omega_\pm$, with probability $1-\delta$, the worst-group error as $t\rightarrow\infty$, is bounded below by:
	\begin{align}\label{eq:thm_lb}
	\Rc_{wst, \rm{LA}}^\infty\geq Q\Big(c\,R_+\,\sqrt{\frac{n}{d}}\,\big(\frac{1}{\tau} + \frac{R_-}{R_+}+ 
	{c_1\sqrt{\frac{\log(n/\delta)}{R_+}}} \big)\Big).
	\end{align}
Recall, $\tau=n_{+1}/n_{-1}\geq 1$ is the imbalance ratio. 
\end{theorem}

Theorem \ref{thm:LALoss} captures the role played by both the imbalance ratio and the spurious correlations in hindering the classification task. Specifically, \eqref{eq:thm_lb} suggests that the worst-group error of the LA-loss increases when either (i) the imbalance ratio $\tau$, or, (ii) the SNR of spurious features, increase. Let us focus on a ``difficult'' classification scenario, where both of the above occur. We will use Theorems \ref{thm:main} and \ref{thm:LALoss} to show that VS-loss can still achieve low error, while the LA-loss fails. 
To be concrete, consider $n=\Theta(d^\eps)$ for  small $\eps>0$ and $R_+=\Theta(d^\beta), \beta\in(1/2,1-\eps)$. Furthermore, let $0\leq R_-/R_+=O(1/\tau)$ and imbalance ratio $\tau = \Omega(R_+\sqrt{n/d})=\Omega(d^{\beta-\frac{1-\eps}{2}})\gg 1.$ Then, from Theorems \ref{thm:main} and \ref{thm:LALoss}, we have that for $\alpha\in(0,1/2-\epsilon):$
\begin{align}\label{eq:alpha}
 \Rc_{wst, \rm{LA}}^\infty \geq  Q(O(1)) \gg Q(\Theta(d^{\alpha})) \geq \Rc_{wst, \rm{VS}}^\infty .
\end{align}
In words, as the dimension $d$ increases, the worst-group error of the LA-loss is lower bounded by an absolute constant, while that of the VS-loss approaches zero. This certifies the critical role of the multiplicative $\Delta_\pm$ factors in the VS-loss, compared to the LA-loss. In the setting above, it is worth noting that VS-loss achieves vanishingly small worst-group error despite extreme imbalance \emph{and} spurious correlations of substantial strength. For example, choosing $n=d^{1/8}, R_+=d^{3/4}, \tau=d^{5/16}$ and $R_-=0$, \eqref{eq:alpha} shows that the worst-group error of the VS-loss for minorities is no worse than $Q(\Theta(d^{1/4}))$ despite having seen only $n_-=\Theta(d^{-\frac{3}{16}})$ training examples from the corresponding minority group, and, despite high spurious-features SNR, i.e. $\|\mub_s\|_2=\|\mub_c\|_2$. 



\vspace{5pt}
\noindent\textbf{$\iota, \omega$ help in early training phase.} It is clear from Theroem \ref{thm:LALoss} and the discussion above that the hyper-parameters $\iota_\pm,\omega_\pm$ of the LA-loss are not sufficient to successfully classify samples from minority groups when spurious features SNR is high. On the other hand, the same parameters play an inconspicuous role in ensuring the favorable performance of the VS-loss as stated in Theorem \ref{thm:main}; specifically, in \eqref{eq:Delta}. The proof of the theorem  makes clear that this condition is introduced to ensure good starting conditions for the GD updates. In particular, the second condition in \eqref{eq:Delta} is essential for our proof, and while we cannot rule out the possibility that this is a technical artifact, we do want to emphasize that the following findings are consistent with what \eqref{eq:Delta} suggests to be necessary. First, the experiments performed in \cite{VSloss} (see also \cite{Autobalance}) confirm that the VS-loss achieves better worst-group error on benchmark datasets (e.g. Waterbirds) when \emph{both} $\Delta_\pm$ and $\iota_\pm$ are tuned. Paraphrasing \cite{VSloss,Autobalance}: ``the hyperparameters $\Delta_\pm$ and $\iota_\pm$ help synergistically in promoting equitable treatment across group''. Complementing their experiments, \cite{VSloss} gave a plausible theoretical justification of the beneficial role of $\iota$'s at the beginning of training by arguing that it favors larger gradients for the minorities at initial epochs. In proving Theroem \ref{thm:main}, we need to balance the gradients between minorities and majorities, which is in fact quite related and can be seen as an alternative justification of the empirically observed beneficial role of $\iota$'s at the initial phase of training. A second worth mentioning outcome of our analysis is that the conditions in \eqref{eq:Delta} are on par with the heuristic strategies proposed in previous literature. Specifically, \cite{Menon,VSloss,Autobalance} recommends tuning $\iota_\pm=-c\log(n_\pm/n)$ 
for some hyperparameter $c>0$. As a final comment, note from Eqn. \eqref{eq:Delta} that $\Delta_\pm$ is set proportional to the corresponding's class frequency $n_\pm/n$. That is, minorities are assigned smaller $\Delta$ values. On the other hand, $\iota$'s (and $\omega$'s) get higher values for minorities. 


\subsection{Learning with label noise}


The proposition below (following similar to Theroem \ref{thm:main}) shows that, even when a (small, but) constant fraction of the labels are flipped, VS-loss can achieve Bayes error with increasing $d$. Thus, it can accomplish \emph{harmless interpolation} / \emph{benign overfitting} of noisy, imbalanced and spuriously correlated data.   

\begin{corollary}[Benign overfitting] In the setting of Theorem \ref{thm:main}, further allow for the possibility that labels are flipped with probability $0<\xi<1/C_0$ before training. Then, provided the VS-loss is tuned as in \eqref{eq:Delta}, the worst-error is bounded with probability at least $1-\delta$ as $\Rc_{wst, \rm{VS}}^\infty\leq \xi + Q\big( c {R_+}/{\sqrt{d}}\big).$

\end{corollary}

As in Theorem \ref{thm:main}, the corollary shows that the VS-loss worst-group error approaches the noise statistical floor provided that $n=O(d^\epsilon)$, $R_+=\Theta(d^\beta)$ for $\beta\in(1/2,1-\epsilon]$ and $d\rightarrow\infty$. Note that in this highly overparameterized regime the training error is zero (in fact the classifier linearly interpolates the data as we show in Section \ref{sec:interpolation}). Hence, this is an instance of benign overfitting, see \cite{bartlett2020benign, chatterji}.

\subsection{Proliferation of support vectors} \label{sec:interpolation}
In a sufficiently overparameterized regime where the good event of Lemma \ref{lem:HP} holds, Lemma \ref{lem:keymain} ensures that the gradient ratio of any two data points is bounded by the inverse sample ratio of their corresponding groups, which is constant with respect to the iteration numbers $t$. This result can also translate to a bound on the respective margins. To see this, denote $\bar{\w} = \frac{\w}{\|\w\|_2}$. Then, for $i,j\in[n]$,
\begin{align*}
|\Delta_{b_i}\z_i^T\bar{\w}^\infty - \Delta_{b_j}\z_j^T\bar{\w}^\infty| &= \lim_{t\to\infty} \frac{|\Delta_{b_i}\z_i^T\w^t-\Delta_{b_j}\z_j^T\w^t|}{\|\w^t\|_2}\\
&\leq \lim_{t\to\infty} \frac{\log(4c_1^2\frac{\Delta_{b_j}^2\, \omega_{b_j}e^{\iota_{b_j}}}{\Delta_{b_i}^2\,\omega_{b_i}e^{\iota_{b_i}}})}{\|\w^t\|_2}\\
&= 0,
\end{align*}
where the inequality follows from Lemma \ref{lem:keymain}, and the last line is true, since by linear separability of data, $\lim_{t\to\infty}\|\w^t\|_2 = \infty$ (see Lemma \ref{lem:GDprop}). Equivalently, 
\begin{align}
	\forall i,j \in [n],\quad  \Delta_{b_i}\z_i^T\bar{\w}^\infty = \Delta_{b_j}\z_j^T\bar{\w}^\infty.& \label{eq:sv}
\end{align}
Thus, all of the samples have the same \emph{scaled} margins $\Delta_{b_i}\z_i^T\bar{\w}$ with respect to the final classifier, and samples within a specific subgroup $(y,a)$, all lie on the same hyperplane.

Furthermore, \cite{VSloss} showed that GD on the VS-loss converges (in direction) to the following cost-sensitive SVM: 
\begin{align}\label{eq:cs-svm}
\w_{\text{CS-SVM}}:=\argmin\|\w\|_2 \quad \text{s.t.} \quad \Delta_{b_i}\z_i^T\w\geq 1,\, i\in[n].
\end{align}
That is, $\bar{\w}^\infty=\bar{\w}_{\text{CS-SVM}}$. On the other hand, from optimality of $\w_{\text{CS-SVM}}$, the constraint is active for at least one of the samples. Combining this with \eqref{eq:sv} certifies that
\begin{align*}
	&\forall i\in[n],\quad \Delta_{b_i}\z_i^T\bar{\w}^\infty = \Delta_{b_i}\z_i^T\bar{\w}_{\text{CS-SVM}}=\frac{1}{\|{\w}_{\text{CS-SVM}}\|_2},
\end{align*}
which implies for all $i \in [n]$,
\begin{align} \label{eq:interpolation}
	&\frac{\|{\w}_{\text{CS-SVM}}\|_2}{\|\w^t\|_2}\cdot\z_i^T\w^t \xrightarrow{t\rightarrow \infty} \z_i ^T \w_{\text{CS-SVM}} = \frac{1}{\Delta_{b_i}}.
\end{align}
In other words, in the setting of Theorem. \ref{thm:main}, all datapoints are support vectors of the cost-sensitive SVM and the gradient descent solution eventually interpolates feature-label pairs $\{(\x_i,y_i/\Delta_{b_i}): i\in[n] \}$. This finding creates a link to the phenomenon of \emph{support vector proliferation}, which is previously shown to occur in overpaprameterized linear settings for the CE loss \cite{wang2021benign,hsu2021proliferation, Vidya, Ke}. Specificaly, we have shown support vector proliferation for the cost-sensitive SVM (and thus for the VS-loss). Interestingly, our proof differs from those in \cite{wang2021benign,hsu2021proliferation} in that we directly analyze the VS-loss updates rather than leveraging first an implicit-bias result and then working directly with the corresponding max-margin classifier. Overall, our proof here is more direct and simpler. On the other hand, it does not yield optimal overparametrization conditions. In particular, extrapolating the result of \cite{Ke} for the SVM, we conjecture that \eqref{eq:interpolation} holds under relaxed condtions $d>Cn\log n$.

\subsection{Extensions} \label{sec:extension}
\noindent \textbf{Logistic loss.} Our analysis considers an exponential version of the VS-loss. The results can be extended to its logistic counterpart, under an additional assumption for the late phase of training. Intuitively, this is made possible because these two loss functions share similar behavior when data points have positive margin. More precisely, the logistic loss can be bounded by exponential loss as follows (see e.g. \cite{rosset2004boosting}):
\begin{align}\label{eq:logloss}
	\z^T\w\geq0 \implies \frac{1}{2}\ell_\text{exp}(\z^T\w) \leq \ell_\text{log}(\z^T\w)\leq \ell_\text{exp}(\z^T\w).
\end{align}
Now, under the event of Lemma \ref{lem:HP}, the data is linearly separable (see Lemma \ref{lem:maxmarg}), and by the implicit bias results of GD \cite{VSloss}, $\w^t$ will converge (in direction) to the solution of \eqref{eq:cs-svm} in this setting. Thus, there exist an iteration $t_0>0$ that for any $t>t_0$, $\w^{t}$ achieves positive margin on all the datapoints, satisfying the condition in \eqref{eq:logloss}. Thus, $\ell_{i,t}$ will behave like an exponential function as $t$ grows. Provided that logistic loss satisfies \eqref{eq:keymain} for a $t_1\geq t_0$, we can show that Lemma \ref{lem:keymain} holds for all $t>t_1$; thus the bound in Theorem \ref{thm:main} will be valid for logistic loss as well. Note that the assumptions on the behavior of the gradient at $t_1\geq t_0$ and on margin positivity essentially restrict us to a \emph{late-phase of training}. Although these conditions are restrictive for this extension, it is worth noting that these are only sufficient conditions for providing a bound on \eqref{eq:bound1} in the proof of Theorem \ref{thm:main}. It is interesting and exciting future work understanding the behavior of VS-loss at earlier training phases.



\vspace{5pt}
\noindent\textbf{Initialization.} Our theorems require gradient descent with zero initialization. Our proof can also be extended to non-zero initialization provided that $\w^0$ lies in some $O(1/\sqrt{d})$ ball. In highly overparameterized settings, this requires the initial weights to remain close to zero. This restriction follows from the proof of Lemma \ref{lem:keymain}, where we need to achieve a lower bound over the ratio of loss derivatives at initialization (see Remark \ref{rem:init}). Similar to previous discussion, this restriction is not necessary for the results to hold, and may be relaxed in future analysis.

\section{Further comparisons to  related works}\label{sec:more-related}

\noindent\textbf{Ref. \cite{VSloss}.} Ref. \cite{VSloss} proposed the VS-loss,  empirically verified its superior performance, and presented a first formal analysis of the distinct (and eventually, synergistic) role of the individual hyperparameters $\iota$ and $\Delta$. Regarding $\Delta$'s, \cite{VSloss} showed their critical role following the two-stage approach described in the Introduction. First comes an optimization-analysis stage, showing the convergence of GD on VS-loss to the cost-sensitive SVM in \eqref{eq:cs-svm}.
The statistical-analysis stage that follows, leverages this connection to derive sharp asymptotics for the worst-case error of the convex program in \eqref{eq:cs-svm}. However, this analysis is inconclusive about $\iota$'s since it focuses at the late (aka converging) phase of training. To argue the benefit of $\iota$'s at the early phase, \cite{VSloss} relied on a separate argument via gradient calculations at initialization. Our analysis of the VS-loss has several distinct features, which we believe bring additional insights. First, 
unlike \cite{VSloss}, our setting allows us to focus explicitly on the effect of spurious correlations (additional to data imbalances). Second, instead of the above two-stage approach, our error bounds follow by directly tracking the GD updates. Interestingly, this analysis reveals simultaneously why both $\iota$'s and $\Delta$'s need to be tuned. Also, our results are non-asymptotic, they hold beyond the so-called proportional asymptotic regime and can be readily extended to sub-Gaussian models. On the other hand, the analysis in \cite{VSloss} is sharp, e.g. allowing them to find the optimal $\Delta$-tuning and study precise tradeoffs between fairness metrics. Also, our analysis regime comes with its own restrictions, e.g. requires $d = \Omega(n^2)$, which we want to relax in future work. 

\vspace{5pt}
\noindent\textbf{Ref. \cite{chatterji}.} On a technical level, our analysis builds on the approach introduced by \cite{chatterji}.
However, while they only considered balanced binary mixtures, we have extended the analysis to imbalanced mixtures with spurious attributes. Also, we have shown that the analysis extends to minimization of a richer  parametric family of the CE loss: the VS-loss. Aside from a more involved data model, the key innovation here lies in proving that once appropriately initialized, the VS-loss ensures that the ratio of loss derivatives between minority/majority samples is inversely proportional to the corresponding ratio of $\Delta$ parameters. The approach of \cite{chatterji} was also recently adapted by \cite{wang2021importance} to study the balanced error of a weighted polynomial loss in a binary imbalanced GMM. Also, the lower bound of Theroem \ref{thm:LALoss} is reminiscent of their result that CE loss can fail in binary classification of imbalanced mixtures. However, our model is more general. Also, their assumptions on scaling of problem parameters are more restrictive than ours (e.g. they require $\|\mub_c\|^2=\Omega(n^2)$ and $d=\Omega(n^3)$).


%% file: proofs-sketch.tex
In this section, we provide a proof sketch of Theorem \ref{thm:main}. Detailed proofs are deferred to the Appendix.
For each $i\in[n]$, denote 
\begin{align*}
	\z_i:=y_i\x_i=\nub_{b_i}+\q_i,\quad\q_i\sim\Nn(0,I_d).
\end{align*}
 Here, recall that $b_i=y_i\cdot a_i$ specifies the group of each sample, and we also let $\nub_{\pm}:=[\mub_c\,;\, \pm\mub_{s}] \in\R^d$ represent the group mean of $\z_i$. 
With this notation and using Assumption \ref{ass:means}, it is easy to show that $$\Rwst^t=Q\big(\min_{b\in\{\pm 1\}}\inp{\nicefrac{\w^t}{\|\w^t\|_2}}{{\nub_{b}}}\big).$$
Hence, our goal is to lower bound the correlation between the GD iterate $\w^t$ and the effective mean vectors $\nub_\pm$. To do this we first need some high probability bounds on the structure of data points, provided by Lemma \ref{lem:HP}.

%
%
%

\begin{lemma}\label{lem:HP}
There exists constant $c_1\geq 1$ and $C_0:=C_0(c_1)>0$ such that for all problem parameters satisfying Assumption \ref{ass:dim}(A-C), and for $b\in\{\pm1\}$, the following statements hold simultaneously with probability $1-\delta$,
\begin{align}
&\forall i\in[n]: d/c_1\leq \|\z_i\|_2^2 \leq c_1 d\\
&{\forall i\in\Tc_{\pm b}\,:\, |\z_i^T\nub_{b}- R_{\pm}| \leq c_1\sqrt{R_+}\sqrt{\log(n/\delta)}
}
\\
&\forall \Tc_{b}\ni i\neq j\in\Tc_{\pm b}\,:|\z_i^T\z_j - {R_\pm}| \leq c_1\sqrt{d}\sqrt{\log(n/\delta)}.
\end{align}
\end{lemma}
The event of this lemma ensures that the data is linearly separable (see Lemma \ref{lem:maxmarg}). 
Conditioned on this event, using the GD update rule and unfolding over iterations, we can show that for all $t\geq 1$:
\begin{align}\label{eq:bound1}
\inp{\frac{\w^t}{\|\w^t\|_2}}{{\nub_{b}}} \geq \frac{R_+}{c_1\sqrt{d}} \cdot \frac{\frac{1}{2} \Lpall + 
\Big(\frac{R_-}{R_+} -  \frac{c_1}{C}\Big)\Lpallneg}{\Lpallmaj+\Lpallmin},
\end{align}
where we set $\Lpall:=\sum_{\tau=0}^t\sum_{i:b_i=b}\ellp_{i,\tau}, b\in\{\pm1\}.$ Note that \eqref{eq:bound1} would directly translate to the bound in \eqref{eq:main_bound} if we showed the second fraction in the RHS of \eqref{eq:bound1} is lower bounded by a constant. To see the challenge, consider minorities $b=-1$ and assume $R_-=0$ (difficult case): we need to show  $(1/2)\Lpallmin-(c_1/C)\Lpallmaj\geq c_0(\Lpallmaj+\Lpallmin)$ for constant $c_0>0$. Intuitively, the challenge is that $\Lpallmin$ is composed of only $n_-$ positive terms, which we would need them to outweigh the $n_+\gg n_-$ positive terms in $\Lpallmaj$. We are able to show this thanks to the following key technical lemma, whose proof is deferred to the Appendix.

\begin{lemma}\label{lem:keymain}
In the event of Lemma \ref{lem:HP}, for all small step sizes, $\forall t\geq 0$ and for all $i,j\in[n]$: If the hyperparameters are all tuned as in the statement of Theorem \ref{thm:main} and Assumption \ref{ass:dim}(D) holds, then 
\begin{align} \label{eq:keymain}
	\frac{\ellp_{i,t}}{\ellp_{j,t}}\leq 4c_1^2 \frac{\Delta_{b_j}}{\Delta_{b_i}} \leq 16c_1^2\frac{n_{b_j}}{n_{b_i}}.
\end{align}
Therefore, there exists small enough constant $c_0>0$ such that $\Lpallmin \geq c_0\Lpallmaj.$
\end{lemma}

This lemma implies that in our high dimensional setting, the contribution of each training sample to the loss gradient is bounded and depends on the size of the group it comes from. This allows us to bound \eqref{eq:bound1} by $cR_+/\sqrt{d}$ to achieve the upper bound of Theorem \ref{thm:main}.

%% file: num.tex

\begin{figure}[!t]
	\centering
	\includegraphics[width=3in]{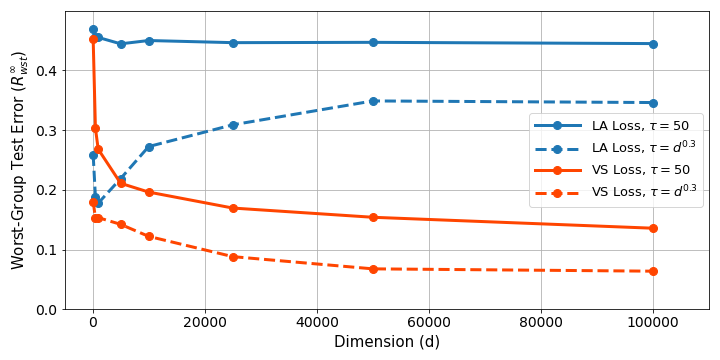}
	\caption{
	Worst-group  error for GMM data (see \eqref{eq:gmm_def}) with $n=200$, $R_+ = d^{0.6}/4$, $\mub_c = \mub_s = \eb_1\sqrt{\nicefrac{R_+}{2}} $ and $\Delta_\pm = n_\pm / n$. 
		The error is evaluated for fixed high imbalance ratio $\tau=50$ (solid lines), and for $\tau$ growing with $d$ (dashed lines). As the model becomes more overparametrized, VS-loss more effectively handles both group-imbalance and spurious correlations compared to the LA-loss.}
	\label{fig:sim}
\end{figure}

Figure \ref{fig:sim} shows results of a numerical simulation performed on synthetic data generated according to \eqref{eq:gmm_def}.
We measured how the worst-group error of both VS-loss and LA-loss evolves with increasing dimension in a difficult scenario as follows. (i) Large imbalance ratio $\tau$: we considered a case of constant imbalance ratio ($\tau=30$) and a case of ratio changing with $d$ ($\tau=d^{0.3}$). (ii) Strong spurious-features SNR: we set $R_+=\Theta(d^{0.6})$ and $R_-=0$, thus the SNR of spurious features is same as the core features. 
For more efficient numeric implementation, instead of running GD on the VS-loss/LA-loss, we evaluated the test error of the cost-sensitive SVM solution in \eqref{eq:cs-svm} as suggested by \cite{VSloss}. For VS-loss, we choose $\Delta_\pm$ proportional to $n_\pm$ to satisfy \eqref{eq:Delta}. For the LA-loss, we choose $\Delta_\pm = 1$, thus \eqref{eq:cs-svm} gives the SVM solution. The simulation results are in alignment with our finding in \eqref{eq:alpha}: The test error of the LA-loss does not decrease with increasing overparametrization, while VS-loss error is vanishingly small. 

For a comparison between LA-loss and VS-loss beyond synthetic data, on real datasets, refer to \cite{VSloss}.
%
%

%% file: conclusion.tex
In this work, we continued the study on VS-loss for handling imbalanced scenarios with spurious features. We investigated the role of hyper-parameters in a Gaussian mixture setting, proving that models trained with VS-loss achieve vanishing worst-group error in sufficiently overparameterized regimes, provided a tuning strategy that is consistent with previous heuristic findings. We further showed that GD on the VS-loss converges to an interpolating solution on the features with a new label encoding that yields larger margin for the minorities. The proof relies on a derivative ratio bound provided in Lemma \ref{lem:keymain}, which requires rather high overparameterization level $d=\Omega(n^2\log(n))$. We note that this is a technical requirement that is not necessary for proving the results and is interesting to relax in future work.

We only considered the binary case, here. However, we believe that the results can be extended to the multi-class case. Another future direction is the extension of our analysis to non-linear settings, see e.g. \cite{frei2022benign}. 
Finally, it is worth noting that this analysis only considers the terminal phase of training. Further investigations on the starting phase can provide better insights on the effects of early stopping on generalization.

%% file: proofs3.tex

This section includes the proof of Theorem \ref{thm:main}. Without further explicit reference, we assume $n$ training data generated as in Assumption \ref{ass:means}. For convenience, we introduce the following notation. For each $i\in[n]$, we denote $\z_i:=y_i\x_i$ and $b_i:=y_i\cdot a_i\in\{\pm1\}$. We also let $\nub_{\pm}:=[\mub_c \, ; \, \pm\mub_{s}]\in\R^d$. 
Then, for all $i\in[n]$, we have 
\begin{align} \label{eq:zmodel}
\z_i:=y_i\x_i=\nub_{b_i}+\q_i,\,\q_i\sim\Nn(0,I_d), i\in[n].
\end{align}
Recall also
\begin{align*}
	&R_+=\|\mub_c\|_2^2+\|\mub_s\|_2^2=\|\nub_{\pm}\|_2^2, \\ &R_-=\|\mub_c\|_2^2-\|\mub_s\|_2^2=\nub_+^T\nub_-,
\end{align*}
and,
$$
\ellp_{i,t}:=\Delta_{b_i}\ell(\x_i^T\w^{t+1},(y_i,a_i))=\Delta_{b_i}\omega_{b_i}e^{-\Delta_{b_i}\z_i^T\w^{t+1}+\iota_{b_i}}.
$$
%
%
In addition, throughout the sections, we use the following shorthand notations for iteration $t$ and $b\in\{\pm1\}$,
\begin{align*}
	\rho_{t,b}&:=\inp{\w^t}{\nub_{b}}\\
	\Lpt&:=\sum_{i:b_i=b}\ellp_{i,t}\\
	\Lpall&:=\sum_{\tau=0}^t\Lptau:=\sum_{\tau=0}^t\sum_{i:b_i=b}\ellp_{i,\tau}\\
	\Lpts&:=\sum_{i=1}^n\ellp_{i,t} = \Lpt + \Lptneg\\
	\Lpalls&:=\sum_{\tau=0}^t\Lptaus:=\sum_{\tau=0}^t\sum_{i=1}^n\ellp_{i,\tau}.
\end{align*} 

\subsection{Preliminary lemmas}


The following lemma shows that for the GMM of Section \ref{sec:GMM}, the per-group risk is a simple function of the correlation between the estimator and the corresponding mean vector.

\begin{lemma}[Worst-group error for GMM] \label{lem:wsterror}
	The per-group classification error of an estimator $\w$ can be bounded in terms of the correlation as follows:
	\begin{align}
	\Rwst=Q\big(\min_{b\in\{\pm 1\}}\inp{\nicefrac{\w}{\|\w\|_2}}{{\nub_{b}}}\big).
	\end{align}
\end{lemma}

\begin{proof}
	We start by the definition of $\Rwst$ from (\ref{eq:worst}) and the data model described in (\ref{eq:zmodel}),
	\begin{align*}
		\Rwst &= \max_{b \in \{\pm 1\}} P_{(\x,a,y) \sim \Pc} \Big(\inp{\nicefrac{\w}{\|\w\|_2}}{\z} < 0 \cond y\cdot a=b\Big)\\
		&= \max_{b \in \{\pm 1\}} P_{\q \sim \Nn(0,I_d)} \Big(\inp{\nicefrac{\w}{\|\w\|_2}}{\nub_{b}} < \inp{\nicefrac{\w}{\|\w\|_2}}{\q} \Big)\\
		&= \max_{b \in \{\pm 1\}} P_{G \sim \Nn(0,1)} \Big(\inp{\nicefrac{\w}{\|\w\|_2}}{\nub_{b}} < G \Big)\\
		&= \max_{b \in \{\pm 1\}} Q \Big(\inp{\nicefrac{\w}{\|\w\|_2}}{\nub_{b}}\Big)\\
		&= Q\big(\min_{b\in\{\pm 1\}}\inp{\nicefrac{\w}{\|\w\|_2}}{{\nub_{b}}}\big)\,.
	\end{align*}
\end{proof}

The next lemma is technical and it provides useful high-probability bounds on several quantities of interest under Assumption \ref{ass:dim}.

\begin{lemma}[Good event $\Ec$]\label{lem:good}
	There exists absolute constant $c_1\geq 1$ and $C_0:=C_0(c_1)>0$ such that for all problem parameters and success probability $1-\delta$ satisfying Assumptions \ref{ass:means} and \ref{ass:dim}, the following statements hold simultaneously for $b \in \{\pm1\}$: 
	\begin{subequations}\label{eq:good event}
		\begin{flalign}
			&\forall i\in[n]\,: \|\z_i\|_2^2 \leq c_1 d
			\label{eq:z_norm_ub}&
			\\
			&\forall i\in[n]\,: \|\z_i\|_2^2 \geq d/c_1
			\label{eq:z_norm_lb}&
			\\
			&\forall i\in\Tc_{b}\,:\,|\z_i^T\nub_b - R_+|\leq  R_+/2
			\label{eq:zmu_same}&
			\\
			&{\forall i\in\Tc_{-b}\,:\, |\z_i^T\nub_{b}- R_{-}| \leq c_1\sqrt{R_+}\sqrt{\log(n/\delta)}
				\label{eq:zmu_nsame_sharp}}&
			\\
			&\forall \Tc_{b}\ni i\neq j\in\Tc_{\pm b}\,:|\z_i^T\z_j - {R_\pm}| \leq c_1\sqrt{d}\sqrt{\log(n/\delta)}.&
			\label{eq:z_align}
		\end{flalign}
	\end{subequations}
	\noindent When all the above hold, we say that the good event $\Ec$ holds.
\end{lemma}
\begin{proof}
	The proof relies on the following concentration bounds for random Gaussian vectors and variables \cite{vershynin2018high}.
	\begin{fact}\label{fact:g}
		For a random variable $G \sim \Nn(0,\sigma^2)$, we have,
		$$
		\Pro \left(\big| G\big| \geq t \right) \leq 2e^{-t^2/(2\sigma^2)}\,.
		$$
		Equivalently, for a fixed probability $\delta>0$, there exists $c_0>0$ such that with probability at least $1-\delta$,
		$$
		\left|G\right| \leq c_0 \sigma \sqrt{\log(1/\delta)}\,.
		$$
	\end{fact}
	\begin{fact}\label{fact:g_norm}
		Let $\boldsymbol g\in\R^d$ have entries i.i.d $\Nn(0,1)$. 
		Then,
		$$
		\Pro\left(\big|\|\boldsymbol g\|_2 - \sqrt{d}\big|>t\right)\leq 2e^{-ct^2/2}.
		$$
		Specifically, for fixed  $\delta$ such that for some constant $C>0$ it holds $d>C\log(1/\delta)$, there exists $c_0\geq 1$ such that with probability at least $1-\delta$,
		$$
		\sqrt{d}/ c_0 \leq  \|\boldsymbol g\|_2 \leq c_0 \sqrt{d}.
		$$
	\end{fact}
\noindent To prove the lemma, it suffices to show that each of the three events listed below will occur with probability at least $1-\delta/3$. Then, taking a union bound over them, we can show that all of the lemma's statements will be true simultaneously with probability at least $1-\delta$.

\vspace{5pt}
\noindent \textbf{(i) Events \eqref{eq:z_norm_ub} and \eqref{eq:z_norm_lb}:} From \eqref{eq:zmodel}, $\z_i-\nub_{b_i} = \q_i$ is a random vector with i.i.d standard Gaussian entries. Based on Fact \ref{fact:g_norm}, for $\delta/(3n)$ and $i\in[n]$, there exists $c_0\geq1$ that $\sqrt{d}/ c_0 \leq  \|\z_i-\nub_{b_i}\|_2 \leq c_0 \sqrt{d}$ with probability at least $1-\delta/(3n)$. Then, there exists $c_1\geq1$,
\begin{align*}
	&\|\z_i\|_2 \leq c_0\sqrt{d} + \sqrt{R_+} \leq c_1\sqrt{d}\\
	&\|\z_i\|_2 \geq \sqrt{d}/ c_0 - \sqrt{R_+} \geq \sqrt{d}/c_1\,,
\end{align*}
where we used the triangle inequality and $d>CR_+$ for some large $C>0$ by Assumption \ref{ass:dim}. Taking a union bound over all samples $i\in[n]$, events \eqref{eq:z_norm_ub} and \eqref{eq:z_norm_lb} occur with probability at least $1-\delta/3$.

\vspace{5pt}
\noindent \textbf{(ii) Events \eqref{eq:zmu_same} and \eqref{eq:zmu_nsame_sharp}:} Note that for any $i\in\Tc_{\pm b}$, $|\z_i^T\nub_b-R_{\pm}|=|\q_i^T\nub_{b}|$, which is the absolute value of a zero-mean Gaussian random variable with variance $R_+$. From Fact \ref{fact:g}, for each sample $i\in[n]$ and fixed probability $\delta$, there exists positive $c_0$ such that  $$\left|\z_i^T\nub_b-R_{\pm}\right| \leq c_0 \sqrt{R_{+}} \sqrt{\log(n/\delta)},$$ with probability at least $1-\delta/(3n)$. Using Assumption \ref{ass:dim}, specifically $R_+\geq\|\mu_c\|_2^2> C\log(n/\delta)$, we can simplify this bound to the RHS of \eqref{eq:zmu_same}. Thus, with probability at least $1-\delta/3$, events \eqref{eq:zmu_same} and \eqref{eq:zmu_nsame_sharp} hold simultaneously for all samples.

\vspace{5pt}
\noindent \textbf{(iii) Event \eqref{eq:z_align}:} For each pair of samples $(i,j)$ with $i\neq j$, $i \in \Tc_b$ and $j \in \Tc_{\pm b}$, we have,
\begin{align}\label{eq:bound3}
	&|\z_i^T\z_j - {R_\pm}| \leq |\q_i^T\nub_{\pm b}| + |\nub_{b}^T\q_j| + |\q_i^T\q_j|\,.
\end{align}
Similar to (ii), we can show that there exists $c_0\geq 0$ such that with probability at least $1 - \delta/6$, the first two terms on the RHS of \eqref{eq:bound3} are upper-bounded by $c_0 \sqrt{R_{+}} \sqrt{\log(n/\delta)}$, simultaneously for all $(i,j)$ pairs.
It now suffices to bound  $|\q_i^T\q_j|$. First, note that we can choose a large enough $c$ such that $\Pro(\|\q_i\|_2\geq c\sqrt{d})\leq\delta/(12n)$. Now, conditioning on $\|\q_i\|_2$, we have,
\begin{align*}
	\Pro(|\q_i^T\q_j|>t) &\leq \Pro\Big(|\q_i^T\q_j|>t\big|\|\q_i\|_2\leq c\sqrt{d}\Big) \\
	&+ \Pro(\|\q_i\|_2\geq c\sqrt{d})\,.
\end{align*}
We can bound the first term by Fact \ref{fact:g} as follows,
\begin{align*}
	\Pro\Big(|\q_i^T\q_j|>t\big|\|\q_i\|_2\leq c\sqrt{d}\Big)&\leq 2e^{-t^2/(2\|q_i\|_2^2)}\\
	& \leq 2e^{-t^2/(2c^2d)}\\
	& = 2e^{-c't^2/d}.
\end{align*}
Now, a union bound on all pairs of samples yields,
\begin{align*}
	\Pro\Big(\exists i\neq j, |\q_i^T\q_j|>t\Big) \leq 2n^2e^{-c't^2/d} + \delta / 12.
\end{align*}
Choosing $t=c'_0\sqrt{d\,\log(n/\delta)}$ for a large enough $c'_0\geq0$, we have,
\begin{align*}
\Pro\Big(\exists i\neq j, |\q_i^T\q_j|>t\Big)&\leq \delta/6.
\end{align*} 
Thus, with probability at least $1-\delta/6$, $|\q_i^T\q_j|<c'_0\sqrt{d\,\log(n/\delta)}$ for all pairs of samples. Overall, the terms on the RHS of \eqref{eq:bound3}, can be bounded at the same time with probability $1-\delta/3$ as follows:
\begin{align*}
	|\z_i^T\z_j - {R_\pm}| &\leq c_0\sqrt{R_+}\sqrt{\log(n/\delta)} + c'_0\sqrt{d}\sqrt{\log(n/\delta)}\\
		&\leq c_1\sqrt{d}\sqrt{\log(n/\delta)}\,,
\end{align*}
where $c_1$ exists due to Assumption \ref{ass:dim}(C). 
\end{proof}

The good event $\Ec$ provides us with useful bounds on the samples lying in the high dimensional setting of Assumption \ref{ass:dim}. More importantly, this high probability event ensures the data is linearly separable, and helps us derive a bound on the maximum margin. This result is formulated in the following lemma.
\begin{lemma}[Data separability and max-margin]\label{lem:maxmarg} 
	For any $\delta\in(0,1)$ and problem parameters that satisfy Assumption \ref{ass:dim}, specifically  $d\geq Cn\|\mub_c\|_2^2 $, $\|\mub_c\|_2\geq C\sqrt{\log(n/\delta)}$ and $n\geq C$ with probability at least $1-\delta$, the training data are linearly separable. Moreover, under the same conditions, with probability $1-\delta$ the maximum margin is lower bounded as follows for some constant $c_2>0$:
	\begin{align}
	\max_{\|\w\|_2=1}\min_{i\in[n]} \inp{\w}{\z_i} \geq c_2\sqrt{\frac{d}{n}}\,.
	\end{align}
\end{lemma}

\begin{proof}
	Recall the following notation for $i\in[n]$:
	$$
	\z_i:=\begin{bmatrix}
	\z_{i,c} \\
	\z_{i,s}
	\end{bmatrix} := y_i\begin{bmatrix}
	\x_{i,c} \\
	\x_{i,s}
	\end{bmatrix} = \begin{bmatrix}
	\mub_c + \q_{i,c}\\
	(y_ia_i)\,\mub_s + \q_{i,s}\,
	\end{bmatrix} ,
	$$
	where $\q_{i,c}$ and $\q_{i,s}$ have i.i.d entries from standard Gaussian distribution. We will prove that 
	$$
	\wt:= \sum_{i\in[n]}\begin{bmatrix}
	\z_{i,c} \\
	0
	\end{bmatrix}
	$$
	is a linear separator under the lemma's assumptions.
	Indeed, we have for any $j\in[n]$:
	\begin{align}
	\inp{\wt}{\z_j} = \|\z_{j,c}\|_2^2 + \inp{\z_{j,c}}{\sum_{i\neq j}\z_{i,c}}.
	\end{align}
	In the sequel, $c_1>1$ is a constant. First, by applying Lemma \ref{lem:good} to $\z_{j,c}$, with probability at least $1-\delta/2$ it holds simultaneously for all $j\in[n]$ that:
	\begin{subequations}\label{eq:event_sep}
		\begin{align}
		&d/c_1 \leq \|\z_{j,c}\|_2^2 \leq c_1 d,\\
		&\inp{\z_{j,c}}{\mub_c}\leq \|\mub_c\|_2^2 + c_1\|\mub_c\|_2\sqrt{\log(n/\delta)}.
		\end{align}
	\end{subequations}
	Second, since $\z_{i,c},\quad i\in[n]$ are i.i.d, we have that $\sum_{i\neq j}\z_{i,c}\sim\Nn((n-1)\mub_c,\sqrt{n-1}I_d)$ and is independent of $\z_{j,c}.$ Hence, conditioned on $\z_{j,c}$, with $ G_j\sim\Nn(0,1)$:
	$$
	\inp{\z_{j,c}}{\sum_{j\neq i}\z_{i,c}} \stackrel{(P)}{=}(n-1)\inp{\z_{j,c}}{\mub_c} + \| \z_{j,c}\|_2 \sqrt{n-1}\cdot G_j.
	$$
	Thus, in the event \eqref{eq:event_sep}, with probability $1-\delta/2$, simultaneously for all $j\in[n]$:
	\begin{align*}
		\inp{\z_{j,c}}{\sum_{j\neq i}\z_{i,c}}\leq& n\|\mub_c\|_2^2 + c_1\|\mub_c\|_2n\sqrt{\log(n/\delta)}\\
		&+c_1\sqrt{d n}\sqrt{\log(n/\delta)}.
	\end{align*}
	
	Putting things together, with probability at least $1-\delta$, simultaneously for all $j\in[n]$, it holds that
	\begin{align*}
	&\inp{\wt}{\z_j} \geq \frac{d}{c_1} 
	\\ &-\Big(n\|\mub_c\|_2^2 + c_1\|\mub_c\|_2n\sqrt{\log(n/\delta)}++c_1\sqrt{d n}\sqrt{\log(n/\delta)}\Big)\,.
	\end{align*}
	By invoking $\sqrt{\log(n/\delta)}\leq \|\mub_c\|_2/C$ and $d/n\geq C \|\mub_c\|_2^2$ for some constant $C>0$, this shows that for some $c_0>1$:
	\begin{align}\label{eq:inp_bound_sep}
	\min_{i\in[n]} \inp{\wt}{\z_i} \geq d/c_0 >0.
	\end{align}
	Therefore, $\wt$ is a linear separator. 
	
	Next, we upper bound the norm of $\wt$. With probability $1-\delta$:
	\begin{align}
	\|\wt\|_2 = \| n\mub_c + \sum_{i\in[n]}\q_{i,c}\| \leq n\|\mub_c\| + c_1\sqrt{n}\sqrt{d}\,,
	\end{align}
	where the upper bound follows by triangle inequality and using that $\sum_{i\in[n]}\q_{i,c}\sim\Nn(0,\sqrt{n} I_d)$ to invoke Fact \ref{fact:g_norm}. Under the conditions of the lemma, $\|\mub_c\|\sqrt{n}\leq \sqrt{d}/\sqrt{C}$ . Hence,  for some $c_0'\geq 1$:
	\begin{align}\label{eq:norm_bound_sep}
	\|\wt\|_2  \leq c_0' \sqrt{n}\sqrt{d}.
	\end{align}
	
	Putting together \eqref{eq:inp_bound_sep} and \eqref{eq:norm_bound_sep}, we conclude with the second statement of the lemma, i.e.
	$$
	\max_{\|\w\|_2=1}\min_{i\in[n]} \inp{\w}{\z_i}  \geq \frac{\min_{i\in[n]}\inp{\wt}{\z_j}}{\|\wt\|_2} \geq c_2 \sqrt{\frac{d}{n}}\,.
	$$
	
\end{proof}

The final lemma of this section describes some important properties of the GD algorithm for a linear model and separable data.
\begin{lemma} [Properties of GD \cite{ji2020gradient}]\label{lem:GDprop}
	Consider loss function $
	\Lc_n(\w) := \frac{1}{n}\sum_{i\in[n]} \ell(\z_i^T\w),
	$ for a convex differentiable strictly decreasing $\ell(.)$ ,
	and GD steps $\{\w^t \}_{t=0}^{\infty}$. Provided that there exists small enough $\eta>0$ that,
	\begin{align}\label{eq:step}
		\Lc_n(\w^{t+1}) - \Lc_n(\w^t) \leq -\frac{\eta}{2}\|\nabla\Lc_n(\w^t)\|_2^2,
	\end{align}
	and training set is linearly separable, $\lim_{t\to\infty} \|\w^t\|_2 = \infty$.
%
\end{lemma}
Smooth (e.g. logistic loss) and locally-smooth functions (e.g. exponential loss) satisfy the step-size condition \eqref{eq:step}\cite{ji2020gradient}. Thus, (logistic / exponential) VS-loss meets the requirements of this lemma. Moreover, from Lemma \ref{lem:maxmarg}, we know that in the high dimensional setting of Assumption \ref{ass:dim}, the data is linearly separable with high probability. Consequently, under good event, we can apply Lemma \ref{lem:GDprop} to the sequence $\{\w^t \}_{t=0}^{\infty}$ defined by \eqref{eq:GD} and we have,
\begin{enumerate}
	\item $\{\Lc_n(\w^t) \}_{t=0}^{\infty}$ is non-increasing.
	\item $\lim_{t\to\infty} \|\w^t\|_2 = \infty$.
\end{enumerate}

\subsection{Key lemmas}


In this section, we present the lemmas that are the key ingredients to the proof of theorems. Lemmas \ref{lem:normbound} and \ref{lem:corrbound} provide bounds on $\|\w^t\|_2$ and $\rho_{t,b}$, that are used to derive inequalities similar to \eqref{eq:bound1}. Lemmas \ref{lem:key}, \ref{lem:ratio2} and \ref{lem:comparison} specify how, by properly tuning the hyperparameters, we can gain control over the sample gradients. At the end of this section, we have all the necessary tools to prove the theorems.
\begin{lemma}[Norm bounds]\label{lem:normbound}
	Under Assumption \ref{ass:dim} and the realization of a good event $\Ec$, for all GD iterations $t>0$ found by \eqref{eq:GD},
	\begin{align}
		\|{\w^{t+1}}\|_2&\leq \|{\w^{0}}\|_2 + \eta \, c_1\,\sqrt{d}\,\Lpalls\,. \label{eq:normUB}
	\end{align}
	Furthermore, when using LA-loss ($\Delta_\pm=1$), there exists $t_0>0$ such that for $t>t_0$, 
	\begin{align}
		\|\w^{t+1}\|_2 &\geq \|\w^{t_0}\|_2 + \frac{\eta c_2}{4} \, \sqrt{\frac{d}{n}} L^\prime_{t_0:t}\,. \label{eq:normLB}
	\end{align}

\end{lemma}
\begin{proof}
	By applying GD iterate (\ref{eq:GD}) together with triangle inequality and (\ref{eq:z_norm_ub}), we have,
	\begin{align*}
		\|\w^{t+1}\|_2&=\|\w^{t}+\eta\sum_{i\in[n]}\ellp_{i,t}\z_i\|_2\\
		&\leq\|\w^{t}\|_2+\eta\sum_{i\in[n]}\ellp_{i,t}{\|\z_i\|_2}
		\\&\leq\|\w^{t}\|_2+\eta c_1 \,\sqrt{d} \sum_{i\in[n]}\ellp_{i,t}\,,
	\end{align*}
	where we also use the fact that $\ellp_{i,t}\geq 0$. Equation \eqref{eq:normUB} follows by repeated application of this inequality.
	
	For the second statement, we start again by GD iterate (\ref{eq:GD}),
	\begin{align}
		\|\w^{t+1}\|_2^2 &= \|\w^{t} + \eta\sum_{i\in[n]} \ellp_{i,t} \z_i\|_2^2 \nn \\
		&\geq \|\w^{t}\|_2^2  + 2\eta\sum_{i\in[n]} \ellp_{i,t} \inp{\z_i}{\w^t}\,. \label{eq:norm}
	\end{align}
	Lemma \ref{lem:maxmarg} ensures that our data is linearly separable and the maximum margin is larger than $c_2\sqrt{d/n}$. On the other hand, GD iterates on the LA-loss converge to the SVM (max margin) solution when the data is separable \cite{VSloss}. Thus, for all large enough $t$,
	\begin{align}
		\inp{\z_i}{\frac{\w^t}{\|\w^t\|_2}} \geq \frac{c_2}{2} \sqrt{\frac{d}{n}}\,. \label{eq:GDmargin}
	\end{align}
	Combining equations \eqref{eq:norm} and \eqref{eq:GDmargin},
	\begin{align} \label{eq:w1}
	\|\w^{t+1}\|^2_2 &\geq \|\w^{t}\|^2_2 + \eta \, c_2 \, \sqrt{\frac{d}{n}} \|\w^t\|_2 L^\prime_t\,\\
	&=\|\w^{t}\|^2_2\Big(1 + \eta \, c_2 \, \sqrt{\frac{d}{n}} \frac{L^\prime_t}{\|\w^{t}\|_2}\Big)\,.
	\end{align}
	From Lemma \ref{lem:GDprop}, we know that the sequence of loss values at $\{\w^t\}_{t=0}^{\infty}$ is non-increasing. Thus,
	\begin{align*}
		\Lpts &= \sum_{i\in[n]} \Delta_{b_i} \ell_{i,t} \leq \Delta_{\max} \sum_{i\in[n]} \ell_{i,t} \leq \Delta_{\max} \sum_{i\in[n]} \ell_{i,0}\,.
	\end{align*}
	where $\Delta_{\max} := \max_{b \in \{\pm 1\}} \Delta_b$. In addition, Lemma \ref{lem:GDprop} ensures $\{\|\w^t\|_2\}_{t=0}^{\infty}$ is unbounded. Hence, when $t$ is large enough,
	\begin{align} \label{eq:w2}
	\|\w^t\|_2 \geq \frac{\eta \, c_2\,\Delta_{\max}}{8} \, \sqrt{\frac{d}{n}} \sum_{i\in[n]} \ell_{i,0} \geq \frac{\eta \, c_2}{8} \, \sqrt{\frac{d}{n}}L^\prime_t\,.
	\end{align}
	We can choose a large enough $t_0$ such that any $t>t_0$ satisfies both \eqref{eq:GDmargin} and \eqref{eq:w2}. Then, by \eqref{eq:w2} and \eqref{eq:w1} we have,
	\begin{align*}
	\forall t\geq t_0:\quad \|\w^{t+1}\|^2_2 &\geq \|\w^{t}\|^2_2\Big(1 + \frac{\eta \, c_2}{4} \, \sqrt{\frac{d}{n}} \frac{L^\prime_t}{\|\w^{t}\|_2}\Big)^2\\
	&=\Big(\|\w^{t}\|_2 + \frac{\eta c_2}{4} \, \sqrt{\frac{d}{n}} L^\prime_t\Big)^2\,.
	\end{align*}
%
	Finally, \eqref{eq:normLB} follows by applying this inequality iteratively for all $t_0\leq\tau\leq t$.
	
\end{proof}



\begin{lemma}[Correlation bounds]\label{lem:corrbound}
	Under Assumption \ref{ass:dim} and conditioned on the good event $\Ec$, for all iterations $t > t_1 \geq 0$ and $b\in\{\pm1\}$, the correlation of GD iterations and group means is bounded as follows,
	\begin{align}
		\rho_{t+1,b} &\geq \rho_{t_1,b}+ {\eta\, R_+}\, \left( \frac{1}{2} L^\prime_{t_1:t,b} + 
		\Big(\frac{R_-}{R_+} -  \frac{c_1}{C}\Big)L^\prime_{t_1:t,-b}\right) \label{eq:corrLB}\,,\\
		\rho_{t+1,b} &\leq \rho_{t_1,b} \nn \\
		& +\eta R_+\left(\frac{3}{2}L^\prime_{t_1:t,b} + \Big(\frac{R_-}{R_+}+c_1\sqrt{\frac{\log(n/\delta)}{R_+}}\Big)L^\prime_{t_1:t,-b}\right). \label{eq:corrUB}
	\end{align}
	
%

\end{lemma}
\begin{proof}
	From GD iteration \eqref{eq:GD} together with \eqref{eq:zmu_same} and \eqref{eq:zmu_nsame_sharp}:
	
	\begin{align*}
		\rho_{t+1,b} &= \rho_{t,b}+\eta\sum_{i: b_i=b}\ellp_{i,t}\inp{\z_i}{\nub_b} + 
		\eta \sum_{i: b_i=-b}\ellp_{i,t}\inp{\z_i}{\nub_{b}} \\
		&\geq \rho_{t,b} +\eta\sum_{i: b_i=b}\ellp_{i,t} \frac{R_+}{2} \\ 
		& \qquad \quad + \eta \sum_{i: b_i=-b}\ellp_{i,t} \big(R_--c_1\sqrt{R_+}\sqrt{\log(n/\delta)} \big)\\
		&\geq \rho_{t,b} +\eta R_+\left(\frac{1}{2}\Lpt + \Big(\frac{R_-}{R_+}-\frac{c_1}{C}\Big)\Lptneg\right).
	\end{align*}
	
	In the last line above, we recalled the assumption $\sqrt{R_+}\geq C \sqrt{\log(n/\delta)}.$
	The first inequality of the lemma follows from repeated application of the inequality in the above display for $\tau=t_1,\ldots,t+1.$ 
	
	The second inequality can be derived similarly by invoking the following inequality.
	\begin{align*}
	\rho_{t+1,b} &= \rho_{t,b}+\eta\sum_{i: b_i=b}\ellp_{i,t}\inp{\z_i}{\nub_b} + 
	\eta \sum_{i: b_i=-b}\ellp_{i,t}\inp{\z_i}{\nub_{b}} \\
	&\leq \rho_{t,b}+\eta\sum_{i: b_i=b}\ellp_{i,t} \frac{3R_+}{2} \\ 
	& \qquad \quad + \eta \sum_{i: b_i=-b}\ellp_{i,t} \big(R_-+c_1\sqrt{R_+}\sqrt{\log(n/\delta)} \big)\,.
	\end{align*}
\end{proof}
\begin{lemma}[Gradients control:VS-loss]\label{lem:key}
	In the good event $\Ec$, for all small step-sizes and $\forall t\geq 0$, the following statement holds for the VS-Loss. If the hyperparameters are all tuned as in the statement of Theorem \ref{thm:main}, then 
	\begin{align}\label{eq:ellp_ratio}
	\max_{i,j\in[n]}\Big\{\frac{\ellp_{i,t}}{\ellp_{j,t}}\Big\}\leq 4c_1^2 \frac{\Delta_{b_j}}{\Delta_{b_i}}.
	\end{align}
\end{lemma}

\begin{proof}
	The conclusion of the lemma is obvious when $i=j$ since $c_1\geq 1$. For convenience let us fix pair of samples $(\x_1,a_1,y_1)$, $(\x_2,a_2,y_2)$ and call  
	$
	A_t := \frac{\ellp_{1,t}}{\ellp_{2,t}}. 
	$
	The same argument holds for any other pair of samples, thus the above restriction is without loss of generality.  
	The proof is based on an induction argument. The base case for $t=0$ and $\w_0=0$, gives: $$A_0= \big({\omega_{b_1}\Delta_{b_1}e^{\iota_{b_1}}}\big)/\big({\omega_{b_2}\Delta_{b_2}e^{\iota_{b_2}}}\big). 
	$$
	Choosing $\omega_{b_1}=\omega_{b_2}=1$ and $\iota_{b_j}= \log(\Delta_{b_j}^{-2}), j=1,2$  yields $A_0\leq \Delta_{b_2}/\Delta_{b_1}\leq 4c_1^2\Delta_{b_2}/\Delta_{b_1},$ since $c_1>1.$

	Next, assume that the statement of the lemma is true for iteration $t$. 
	Our starting point to prove the statement for $t+1$ is the following recursion, obtained directly from the GD update rule:
	\begin{align}
	&\log \frac{A_{t+1}}{A_t} =  -\eta\Delta_{b_1}\|\z_1\|_2^2\ellp_{1,t}+\eta\Delta_{b_2}\|\z_2\|_2^2\ellp_{2,t} \nn\\
	& -\eta\Delta_{b_1}\sum_{1\neq i\in[n]}\ellp_{i,t}\inp{\z_i}{\z_1}+
	\eta\Delta_{b_2}\sum_{2\neq i\in[n]}\ellp_{i,t}\inp{\z_i}{\z_2} \nn
	\\
	&\leq -\eta\Delta_{b_2}\ellp_{2,t} \|\z_1\|_2^2\,\Big(\frac{\ellp_{1,t}}{\ellp_{2,t}} \frac{\Delta_{b_1}}{\Delta_{b_2}} - \frac{\|\z_2\|_2^2}{\|\z_1\|_2^2}\Big) \nn
	\\& + \eta\Delta_{b_2}\ellp_{2,t} \big(\frac{\Delta_{b_1}}{\Delta_{b_2}}S_{1,+}+\frac{\Delta_{b_1}}{\Delta_{b_2}}S_{1,-}+S_{2,+}+S_{2,-}\big), \label{eq:recursion}
	\end{align}
	where we set for  $j=1,2$, 
	$
	S_{j,\pm} = \sum_{i\neq j:b_i=\pm b_j}\frac{\ellp_{i,t}}{\ellp_{2,t}}|\inp{\z_i}{\z_j}|.
	$
	Now, note from applying \eqref{eq:z_align} that 
	\begin{align}
		\max_{i\neq j:b_i=\pm b_j} |\inp{\z_i}{\z_j}| \leq U, \label{eq:maxcorr}
	\end{align}
	where, we set $U:=R_++c_1\sqrt{d}\sqrt{\log(n/\delta)}$ and we used that $R_+\geq R_-.$  Moreover, using the induction hypothesis for iteration $t$:
	\begin{align}
	\sum_{i\neq j:b_i=\pm b_j}\frac{\ellp_{i,t}}{\ellp_{2,t}} \leq  n_{\pm b_j}  \cdot \max_{i\in \Tc_{\pm b_j}}\frac{\ellp_{i,t}}{\ellp_{2,t}} \nn
	&\leq n_{\pm b_j}  \cdot 4c_1^2 \, \frac{\Delta_{b_2}}{\Delta_{\pm b_j}}. \nn
	\end{align}
	
	Combining, the above three displays, we have shown that
	\begin{align}
	&\frac{\Delta_{b_1}}{\Delta_{b_2}}S_{1,+}+\frac{\Delta_{b_1}}{\Delta_{b_2}}S_{1,-}+S_{2,+}+S_{2,-} 
	\\
	\nn&\leq  4c_1^2\,U\big( n_{b_1} +  n_{-b_1} \frac{\Delta_{b_1}}{\Delta_{-b_1}} +  n_{b_2} +  n_{-b_2} \frac{\Delta_{b_2}}{\Delta_{-b_2}}\big) 
	\\ 
	&\leq 4c_1^2\,  U\big(n_{b_1} + 4 n_{b_1} + n_{b_2} + 4 n_{b_2} \big) 
	\leq 40c_1^2\,U\,n. \label{eq:2nd_bracket}
	\end{align}
	To obtain the penultimate inequality, we used the tuning of $\Delta$'s in \eqref{eq:Delta}. 
	
	Now, using \eqref{eq:2nd_bracket} and the fact that $d/c_1\leq \|\z_1\|_2^2\leq c_1d$, we conclude with the following bound:
	\begin{align}
	&\log\frac{A_{t+1}}{A_{t}} \leq 
	-\frac{\eta}{c_1}\Delta_{b_2}\ellp_{2,t}\, d\,\Big(
	\frac{\Delta_{b_1}\ellp_{1,t}}{\Delta_{b_2}\ellp_{2,t}} - c_1^2-40c_1^3\,\frac{U\,n}{d}
	\Big)\,.\label{eq:log_1}
	\end{align}
	At this point, we invoke Assumption \ref{ass:dim}(C-D) with a sufficiently large $C_0$. Then, for all $C>C_0$, it holds that $ Un/d \leq 1/(40 \, c_1).$ Therefore, from \eqref{eq:log_1}:
	\begin{align}
	&\log\frac{A_{t+1}}{A_{t}} \leq 
	-\frac{\eta}{c_1} \Delta_{b_2}\ellp_{2,t}\, d\,\Big(
	\frac{\Delta_{b_1}\ellp_{1,t}}{\Delta_{b_2}\ellp_{2,t}} -2c_1^2
	\Big)\label{eq:log_2}.
	\end{align}
	Finally, we distinguish two cases as follows. 
	
	If $A_t\geq 2 c_1^2\Delta_{b_2}/\Delta_{b_1}$, then the parenthesis in the RHS of \eqref{eq:log_2} is positive. Hence, in this case $A_{t+1}\leq A_t\leq 4c_1^2\Delta_{b_2}/\Delta_{b_1}$ as desired.
	
	If the above does not hold, then $A_t< 2 c_1^2\Delta_{b_2}/\Delta_{b_1}$.  In this case \eqref{eq:log_2} can be further upper bounded as follows:
	\begin{align}
	&\log\frac{A_{t+1}}{A_{t}} \leq  \frac{2c_1}{1}\,{\eta}\, \Delta_{b_2}\,\ellp_{2,t}\, d \leq 8c_1\,\eta\,d\sum_{i\in[n]}\ell_{i,t} \leq 8c_1\,\eta\,dn\,. \label{eq:case2}
	\end{align}
	In the second inequality, we recalled that $\ellp_{2,t}=\Delta_{b_2}\ell_{2,t}$ and also used \eqref{eq:Delta}. The third inequality follows because the total loss is non-increasing at every $t$: $\sum_{i\in[n]}\ell_{i,t}\leq \sum_{i\in[n]}\ell_{i,0}=n$ (Lemma \ref{lem:GDprop}). Now, choosing small enough step-size $\eta\leq \log(2)/(8c_1\,d\,n)$, we find that $\log\frac{A_{t+1}}{A_{t}}\leq \log 2$, which implies $A_{t+1}\leq 2A_t\leq 4 c_1^2\Delta_{b_2}/\Delta_{b_1}$. 
	
	In both cases, the desired inequality is proven.
\end{proof}

\begin{remark}[Non-zero initialization] \label{rem:init}
	In the theorem, we have assumed zero initialization to ensure that in the induction argument of Lemma \ref{lem:key}, the base case holds. We can also extend the proof to non-zero initialization provided that we stay close to zero. Specifically, assume $\|\w^0\|_2\leq c_0/\sqrt{d}$ for some positive constant $c_0$. Then, using the tuning strategy in \eqref{eq:Delta} and for $t=0$, we have
	\begin{align*}
		\log \Big(\frac{\Delta_{b_1}}{\Delta_{b_2}}A_0\Big) &= \inp{\w^0}{\Delta_{b_2}\z_2 - \Delta_{b_1}\z_1}\\
		&\leq \|\w^0\|_2 \cdot \big(\Delta_{b_2}\|\z_2\|_2 + \Delta_{b_1}\|\z_1\|_2\big) \\
		&\leq c_0c_1(\Delta_{b_2}+\Delta_{b_1})\\
		&\leq 4c_0c_1,
	\end{align*}
	where the first inequality uses Cauchy-Schwarz and triangle inequality, the second line follows from the assumption on $\|\w^0\|_2$ and \eqref{eq:z_norm_ub}, and the last line is true since by \eqref{eq:Delta}, $\Delta_\pm \leq 2$. Thus,
	\begin{align*}
		A_0 \leq c_3 \frac{\Delta_{b_2}}{\Delta_{b_1}},
	\end{align*}
	for $c_3 = e^{4c_0c_1}$, which is a positive constant. With a similar argument as in the proof of the lemma, we can prove by induction that for all $t\geq0$, $A_t \leq c_3 \Delta_{b_2}/\Delta_{b_1}$. Having the extension of Lemma \ref{lem:key} in hand, the rest of the proof of the theorem is the same for both initializations. In other words, the results hold for non-zero initial weights provided that $\w^0$ lies in a $O(1/\sqrt{d})$ ball.
\end{remark}

\begin{lemma}[Gradient control: LA-loss]\label{lem:ratio2}
	When a good event $\Ec$ holds, with sufficiently small step-size, the following inequality is satisfied for all GD steps of the LA-loss,
	\begin{align}
	\max_{i,j\in[n]}\Big\{\frac{\ellp_{i,t}}{\ellp_{j,t}}\Big\}\leq c_3\,,
	\end{align} 
	where 
	\begin{align}
	c_3 = \max \{4c_1^2, \frac{\omega_+ e^{\iota_+}}{\omega_- e^{\iota_-}}, \frac{\omega_- e^{\iota_-}}{\omega_+ e^{\iota_+}}\}\,.	\label{eq:lossratio_cond}
	\end{align}
\end{lemma}
\begin{proof}
	Similar to the previous lemma, we can prove this statement by induction. Again, without loss of generality, we fix the pair of samples. The base case clearly holds,
	\begin{align*}
		A_0 = \big({\omega_{b_1}e^{\iota_{b_1}}}\big)/\big({\omega_{b_2}e^{\iota_{b_2}}}\big) \leq c_3.
	\end{align*}
	Now suppose the statement is true for iteration $t$. Note that LA-loss is equivalent to VS-loss when $\Delta_\pm=1$. To prove the statement for $t+1$, we can start from \eqref{eq:recursion} by setting $\Delta_{b_1}=\Delta_{b_2}=1$.
	
	\begin{align}
		\log \frac{A_{t+1}}{A_t} &\leq -\eta\ellp_{2,t} \|\z_1\|_2^2\,\Big(\frac{\ellp_{1,t}}{\ellp_{2,t}} - \frac{\|\z_2\|_2^2}{\|\z_1\|_2^2}\Big) \nn 
		\\& + \eta\ellp_{2,t} \big(\sum_{1\neq i\in[n]}\frac{\ellp_{i,t}}{\ellp_{2,t}}|\inp{\z_i}{\z_1}| + \sum_{2\neq i\in[n]}\frac{\ellp_{i,t}}{\ellp_{2,t}}|\inp{\z_i}{\z_2}|\big),\nn \\
		& \leq -\frac{\eta}{c_1}\ellp_{2,t}\, d\,\Big(
		\frac{\ellp_{1,t}}{\ellp_{2,t}} - c_1^2-2c_1c_3\,\frac{U\,n}{d}
		\Big)\nn \\
		& \leq -\frac{\eta}{c_1}\ellp_{2,t}\, d\,\Big(
		\frac{\ellp_{1,t}}{\ellp_{2,t}} - \frac{c_3}{4}-2c_1c_3\,\frac{U\,n}{d}
		\Big)\nn \\
		&\leq -\frac{\eta}{c_1}\ellp_{2,t}\, d\,\Big(
		\frac{\ellp_{1,t}}{\ellp_{2,t}} - \frac{c_3}{2}\Big)\,. \label{eq:ratio2}
	\end{align}
	Here, we used \eqref{eq:maxcorr} to derive the second inequality. For the last inequality, we choose a large enough $C_0$ and apply Assumption \ref{ass:dim}(C-D). Then, we have $Un/d \leq 1/(8c_1)$ for all $C>C_0$. To finalize the proof, we consider two cases. If $A_t \geq c_3/2$, the RHS of \eqref{eq:ratio2} is negative and $A_{t+1} \leq A_t \leq c_3$. Otherwise, similar to \eqref{eq:case2}, for small enough step size we have,
	\begin{align*}
		\log \frac{A_{t+1}}{A_t} \leq \eta\frac{c_3}{2c_1} d \sum_{i\in[n]} \ell_{i,t} \leq \eta\frac{c_3}{2c_1} d n \leq \log(2)\,.
	\end{align*}
	Thus, $A_{t+1} \leq 2A_t \leq c_3$, which completes the proof. 
\end{proof}

\begin{lemma}\label{lem:comparison}
	Assume \eqref{eq:ellp_ratio} is true for all $t$. Then, there exists constant $c_0>0$ such that
	\begin{align*}
	c_0 \,\Lpts \leq 
	\frac{1}{2}\Lpt +\Big(\frac{R_-}{R_+}-\frac{c_1}{C}\Big)\Lptneg.
	\end{align*}
\end{lemma}
\begin{proof}
	For simplicity denote $E:=\frac{R_-}{R_+}-\frac{c_1}{C}$.
	It suffices to show that
	$
	\big(1/2-c_0 \big)\Lpt \geq \big(c_0  - E\big)\Lptneg, 
	$
	or simpler, assuming $1/2-c_0>0$,
	\begin{align*}
	&\frac{c_0-E}{1/2-c_0} \leq \frac{n_{b}}{n_{-b}}\,\frac{\min_{i\in\Tc_{b}}\ellp_{t,i}}{\max_{i\in\Tc_{-b}}\ellp_{t,i}} \,.
	\end{align*}
	Recall now from \eqref{eq:ellp_ratio} and \eqref{eq:Delta} that
	\begin{align*}
	&\frac{\max_{i\in\Tc_{-b}}\ellp_{t,i}}{\min_{i\in\Tc_b}\ellp_{t,i}}\leq 4c_1^2\frac{\Delta_{b}}{\Delta_{-b}}\leq 16c_1^2 \frac{n_b}{n_{-b}}.
	\end{align*}
	Hence, it suffices that
	\begin{align*}
	&\frac{c_0-E }{1/2-c_0} \leq \frac{1}{16c_1^2} \Leftarrow c_0 \leq \frac{8c_1^2-16c_1^3/C+16c_1^2(R_-/R_+)}{1+16c_1^2}.
	\end{align*}
	Provided that $R_-/R_+\geq0$ and since $C$ can be set large enough (say relative to $c_1$), we can always choose positive $c_0<1/2$ satisfying the above.
\end{proof}


\subsection{Proof of Thereom \ref{thm:main}}
Now we are ready to prove Theorem \ref{thm:main}. From Lemma \ref{lem:wsterror}, the error of group $b\in\{\pm1 \}$ is,
\begin{align*}
\Rc_{\text{wst}, \rm{VS}}^\infty=  Q\big(\inp{\nicefrac{\w^\infty}{\|\w^\infty\|_2}}{{\nub_{b}}}\big)
&=\lim_{t\to\infty}Q\big(\inp{\nicefrac{\w^t}{\|\w^t\|_2}}{{\nub_{b}}}\big) .
\end{align*}
To bound the worst-group error, we need to bound the error for both groups above. Equivalently, it suffices to find a lower bound on $\rho_{t,b} = \inp{\w^t}{\nub_{b}}$ and an upper bound on  $\|\w^t\|_2$. Based on Lemmas \ref{lem:normbound} and \ref{lem:corrbound}, for $t_1=0$, we have,
\begin{align*}
\inp{\nicefrac{\w^t}{\|\w^t\|_2}}{{\nub_{b}}} &= \frac{\rho_{t,b}}{\|\w^t\|_2}\\
& \geq \frac{\rho_{0,b}+ {\eta\, R_+}\, \left( \frac{1}{2} \Lpall + 
	\Big(\frac{R_-}{R_+} -  \frac{c_1}{C}\Big)\Lpallneg\right)}{\|{\w^{0}}\|_2 + \eta \, c_1\,\sqrt{d}\,\Lpalls}\\
&\geq \frac{\rho_{0,b}+ {\eta\,c_0\, R_+\,\Lpalls}\,}{\|{\w^{0}}\|_2 + \eta \, c_1\,\sqrt{d}\,\Lpalls},
\end{align*} 
where in the last line we chose $c_0$ according to Lemma \ref{lem:comparison}. Based on Lemma \ref{lem:GDprop}, we know that $\|\w^t\|_2\to \infty$. Since $\|\w^t\|_2$ is bounded by \eqref{eq:normUB}, $L^\prime_{0:\infty} \to \infty$. Thus we have,
\begin{align*}
\lim_{t\to\infty}\inp{\nicefrac{\w^t}{\|\w^t\|_2}}{{\nub_{b}}} &\geq \lim_{t\to\infty} \frac{\rho_{0,b}+ {\eta\,c_0\, R_+\,\Lpalls}\,}{\|{\w^{0}}\|_2 + \eta \, c_1\,\sqrt{d}\,\Lpalls} \\
&= c\frac{R_+}{\sqrt{d}},
\end{align*}
for $c=c_0/c_1$. Finally, since the Q-function is decreasing,
\begin{align*}
\Rc_{wst, \rm{VS}}^\infty &\leq Q\big( c {R_+}\big/{\sqrt{d}}\big)\, .
\end{align*}


\subsection{Proof of Thereom \ref{thm:LALoss}}


Similar to the previous section, we start with the definition of the worst group error,
\begin{align*}
	\Rc_{\text{wst}, \rm{LA}}^\infty &= Q\big(\min_{b\in\{\pm 1\}}\inp{\nicefrac{\w^\infty}{\|\w^\infty\|_2}}{{\nub_{b}}}\big)\\
	&= \max_{b \in \{\pm 1\}} Q\big(\inp{\nicefrac{\w^\infty}{\|\w^\infty\|_2}}{{\nub_{b}}}\big).
\end{align*}
We will provide a lower bound on the value of the objective function for the minority group (i.e. $b=-1$), which will also bound $\Rc_{\text{wst}, \rm{LA}}^\infty$.
From Lemmas \ref{lem:normbound} and \ref{lem:corrbound}, for $t>t_0$,
\begin{align}
	&\inp{\nicefrac{\w^t}{\|\w^t\|_2}}{{\nub_{b}}} = \frac{\rho_{t,b}}{\|\w^t\|_2}\nn\\
	& \leq \frac{\rho_{t_0,b} +\eta R_+\left(\frac{3}{2}L^\prime_{t_0:t,b} + \Big(\frac{R_-}{R_+}+c_1\sqrt{\frac{\log(n/\delta)}{R_+}}\Big)L^\prime_{t_0:t,-b}\right)}{\|\w^{t_0}\|_2 + \frac{\eta c_2}{4} \, \sqrt{\frac{d}{n}} L^\prime_{t_0:t}}\nn\\
	& \leq \frac{\rho_{t_0,b} +\eta c_4 \, R_+\left(\frac{n_b}{n_{-b}} + \frac{R_-}{R_+}+c_1\sqrt{\frac{\log(n/\delta)}{R_+}} \right)L^\prime_{t_0:t}}{\|\w^{t_0}\|_2 + \frac{\eta c_2}{4} \, \sqrt{\frac{d}{n}} L^\prime_{t_0:t}}\,,  \label{eq:finalbound}
	\end{align}
with $t_0$ defined in Lemma \ref{lem:normbound}. In the last inequality, we used Lemma \ref{lem:ratio2}, specifically $\ellp_{i,t}\leq\frac{c_3}{n}\sum_{j\in[n]} \ellp_{j,t}$, to bound the nominator. 

From Lemma \ref{lem:normbound}, we have $\norm{\w^t}_2 \leq \norm{\w^{t_0}}_2 + \eta L^\prime_{t_0:t} $. Since $\norm{\w^t}_2$ is unbounded, $L^\prime_{t_0:t} \rightarrow \infty$ as $t$ gets large. Hence, in \eqref{eq:finalbound}, the terms containing $L^\prime_{t_0:t}$ will dominate and for large $t$,
\begin{align*}
\lim_{t\to\infty} \frac{\inp{\w^t}{\nub_{b}}}{\|\w^t\|_2} &\leq cR_+\sqrt{\frac{n}{d}}\Big( \frac{n_b}{n_{-b}} + \frac{R_-}{R_+} + c_1\sqrt{\frac{\log(n/\delta)}{R_+}} \Big).
\end{align*}
For $b=-1$, $\frac{n_b}{n_{-b}}=\frac{1}{\tau}$. Eventually, the worst-group error can be bounded as follows, which completes the proof.
\begin{align*}
\Rc_{wst, \rm{LA}}^\infty &\geq Q\big(\inp{\nicefrac{\w^\infty}{\|\w^\infty\|_2}}{{\nub_{-1}}}\big)\\
&\geq Q\Bigg(c\,R_+\,\sqrt{\frac{n}{d}}\,\Big(\frac{1}{\tau} + \frac{R_-}{R_+}+ 
{c_1\sqrt{\frac{\log(n/\delta)}{R_+}}} \Big)\Bigg).
\end{align*}